%% file: ms.tex
\def\year{2021}\relax

\documentclass[letterpaper]{article} % DO NOT CHANGE THIS
\usepackage{aaai21}  % DO NOT CHANGE THIS
\usepackage{times}  % DO NOT CHANGE THIS
\usepackage{helvet} % DO NOT CHANGE THIS
\usepackage{courier}  % DO NOT CHANGE THIS
\usepackage[hyphens]{url}  % DO NOT CHANGE THIS
\usepackage{graphicx} % DO NOT CHANGE THIS
\usepackage{natbib}  % DO NOT CHANGE THIS AND DO NOT ADD ANY OPTIONS TO IT
\usepackage{caption} % DO NOT CHANGE THIS AND DO NOT ADD ANY OPTIONS TO IT
\usepackage[utf8]{inputenc} 
\usepackage[T1]{fontenc}    
\usepackage{url}            
\usepackage{booktabs}       
\usepackage{amsfonts}       
\usepackage{nicefrac}     
\usepackage{microtype}     
\usepackage{amsthm}

\newtheorem{lemma}{Lemma}[section]

\theoremstyle{definition}

\usepackage{amsmath}
\usepackage{amsfonts}
\usepackage{pdfpages}
\usepackage{standalone}
\usepackage{paralist}

\usepackage[inline]{enumitem}

\input{Commands.tex}

\def\compileFigures{0}
\newcommand{\filename}{main}
\newcounter{figuerNumber}

\input{ColorsAndPlots}

\usepackage{graphicx}
\usepackage{xspace}
\usepackage{todonotes}

\title{On the Privacy Risks of Model Explanations}
\author {
	Reza Shokri,\textsuperscript{\rm 1}
	Martin Strobel, \textsuperscript{\rm 1}
	Yair Zick \textsuperscript{\rm 2} \\
}
\affiliations {
	\textsuperscript{\rm 1} National University of Singapore \\
	\textsuperscript{\rm 2} University of Massachusetts, Amherst \\
	reza@comp.nus.edu.sg, mstrobel@comp.nus.edu.sg, yzick@umass.edu
}

\begin{document}
	
	\maketitle
	
	\begin{abstract}
		Privacy and transparency are two key foundations of trustworthy machine learning.  Model explanations offer insights into a model's decisions on input data, whereas privacy is primarily concerned with protecting information about the training data.  We analyze connections between model explanations and the leakage of sensitive information about the model's training set. We investigate the privacy risks of feature-based model explanations using {\em membership inference attacks}: quantifying how much model predictions plus their explanations leak information about the presence of a datapoint in the training set of a model. We extensively evaluate membership inference attacks based on feature-based model explanations, over a variety of datasets.  We show that backpropagation-based explanations can leak a significant amount of information about individual training datapoints.  This is because they reveal statistical information about the decision boundaries of the model about an input, which can reveal its membership. We also empirically investigate the  trade-off between privacy and explanation quality, by studying the perturbation-based model explanations.
	\end{abstract}
	
	\input{Introduction}
	\input{PreliminariesAndProblemformulation}
	\input{MembershipInferenceAttacks}
	\input{ExperimentsForGradientBasedExplanations}
	\input{DiscussionPerturbationBasedExplanations}

	\input{BroaderImpact}
	\input{Conclusion}
	
	\bibliography{abb,literature} 
	
	\clearpage
	\newpage
	\appendix
	\input{RecordExperiments}
	\input{appendices/Appendix_MembershipInference_ExampleBased}

	\input{RecordAnalysis}

\input{DatasetReconstruction}

	\input{appendices/Appendix_RecordBaselineReconstruction}

\end{document}

%% file: Commands.tex
\usepackage{comment}
\usepackage{tcolorbox}
\usepackage{algpseudocode} 
\usepackage{paralist}
\usepackage{algorithm}
\usepackage{multirow}
\usepackage{thmtools} 
\usepackage{thm-restate}
\usepackage{amssymb}
\specialcomment{yaircomment}{\begin{tcolorbox}[colback=red!5!white,colframe=red,title=Yair:]
		\begingroup\ttfamily}{\endgroup\end{tcolorbox}}
\specialcomment{martincomment}{\begin{tcolorbox}[colback=blue!5!white,colframe=blue,title=Martin:]
		\begingroup\ttfamily}{\endgroup\end{tcolorbox}}
\specialcomment{rezacomment}{\begin{tcolorbox}[colback=green!5!white,colframe=green,title=Reza:]
		\begingroup\ttfamily}{\endgroup\end{tcolorbox}}
\renewcommand{\cal}[1]{\mathcal{#1}}
\newcommand{\R}{\mathbb{R}}
\DeclareMathOperator{\argmin}{\mathrm{argmin}}
\DeclareMathOperator{\argmax}{\mathrm{argmax}} 
\newcommand{\train}{\text{train}}
\newcommand{\LIME}{\text{LIME}}

\newcommand{\GRAD}{\mathit{GRAD}}
\newcommand{\SMOOTH}{\mathit{SMOOTH}}
\newcommand{\INTGRAD}{\mathit{INTG}}
\newcommand{\LRP}{\mathit{LRP}}

\newcommand{\euler}{\mathrm{e}}

\newcommand{\texas}{Texas\xspace}
\newcommand{\purchase}{Purchase\xspace}
\newcommand{\hospital}{Hospital\xspace}
\newcommand{\fishdog}{Dog/Fish\xspace}
\newcommand{\adult}{Adult\xspace}

\newcommand{\cifarTen}{CIFAR-10\xspace}
\newcommand{\cifarHundred}{CIFAR-100\xspace}

\newcommand{\Loss}{\mathit{Loss}}
\newcommand{\Prediction}{\mathit{Pred}}
\newcommand{\Explanation}{\mathit{Expl}}

\renewcommand{\cal}[1]{\mathcal{#1}}
\newcommand{\Var}{\mathit{Var}}
\newcommand{\poi}{\vec y}

\renewcommand{\exp}{\mathrm{e}}

\newcommand{\linearlyIndependent}{\texttt{linearlyIndependent}}
\newcommand{\True}{\textit{True}}
\newcommand{\False}{\textit{False}}
\newcommand{\revealedVectors}{\texttt{revealedVectors}}
\newcommand{\pickRandomPointIn}{\texttt{pickRandomPointIn}}
\newcommand{\dotappend}{\texttt{.append}}
\newcommand{\vectorsNotFound}{\texttt{vectorsNotFound}}
\newcommand{\getBasisOf}{\texttt{getBasisOf}}
\newcommand{\spanningVectors}{\texttt{spanningVectors}}

%% file: ColorsAndPlots.tex
\usepackage{tikz}
\usetikzlibrary{arrows.meta}
\usetikzlibrary{calc}
\usetikzlibrary{shapes.misc}
\usetikzlibrary{patterns}

\if\compileFigures1
\usetikzlibrary{external}
\fi

\usepackage{pgfplots}
\usepackage{pgfplotstable}
\usepgfplotslibrary{groupplots}

\usepackage{graphicx}
\usepackage{subcaption}

\usepackage{color}
\usepackage{colortbl}
\definecolor{grey}{rgb}{0.7,0.7,0.7}
\definecolor{darkgrey}{rgb}{0.5,0.5,0.5}
\definecolor{bgreen}{rgb}{0.7,1,0.7}

\definecolor{blue}{RGB}{68,1,84}
\definecolor{bblue}{RGB}{72,35,116}
\definecolor{red}{RGB}{32,144,140}
\definecolor{bred}{RGB}{34,167,13}
\definecolor{green}{RGB}{189,222,38}
\definecolor{bgreen}{RGB}{253,231,36}

 \definecolor{cmap1}{RGB}{0,0,3}
		\definecolor{cmap2}{RGB}{20,13,53}
		\definecolor{cmap3}{RGB}{59,15,111}
		\definecolor{cmap4}{RGB}{99,25,127}
		\definecolor{cmap5}{RGB}{140,41,128}
		\definecolor{cmap6}{RGB}{182,54,121}
		\definecolor{cmap7}{RGB}{221,73,104}
		\definecolor{cmap8}{RGB}{246,112,91}
		\definecolor{cmap9}{RGB}{253,159,108}
		\definecolor{cmap10}{RGB}{253,207,146}
		\definecolor{cmap11}{RGB}{251,252,191}

\pgfplotsset{
colormap={mygreen}{rgb255(0cm)=(254,254,254);rgb255(1cm)= (4,101,53)},
colormap={myblue}{rgb255(0cm)=(160,30,50); rgb255(1cm)=(39,59,129)},
colormap={myred}{rgb255(0cm)=(254,254,254); rgb255(1cm)=(160,30,50)},
colormap={mybred}{rgb255(0cm)=(254,254,254); rgb255(1cm)=(217,10,100)},
colormap={mybblue}{rgb255(0cm)=(217,10,100); rgb255(1cm)=(10,157,217)},
colormap={mybgreen}{rgb255(0cm)=(254,254,254); rgb255(1cm)=(98,159,67)}
}

%% file: Introduction.tex
\section{Introduction}\label{sec:introduction}

Black-box machine learning models are often used to make high-stakes decisions in sensitive domains. However, their inherent complexity makes it extremely difficult to understand the \emph{reasoning} underlying their predictions.
This development has resulted in increasing pressure from the general public and government agencies; several proposals advocate for deploying (automated) {\em model explanations}~\cite{goodman2017explanation}.  In recent years, novel explanation frameworks have been put forward; Google, Microsoft, and IBM now offer model explanation toolkits as part of their ML suites.\footnote{See \url{http://aix360.mybluemix.net/}, \url{https://aka.ms/AzureMLModelInterpretability} and \url{https://cloud.google.com/explainable-ai}.}  

Model explanations offer users additional information about how the model made a decision with respect to their data records. Releasing additional information is, however, a risky prospect from a privacy perspective.  The explanations, as functions of the model trained on a private dataset, might inadvertently leak information about the training set, beyond what is necessary to provide useful explanations.  Despite this potential risk, there has been little effort to analyze and address any data privacy concerns that might arise due to the release of model explanations.  This is where our work comes in. We initiate this line of research by asking the following question: \textbf{can an adversary leverage model explanations to infer private information about the training data?}

The established approach to analyze information leakage in machine learning algorithms is to take the perspective of an adversary and design an attack that recovers private information, thus illustrating the deficiencies of existing algorithms (e.g., \cite{aivodji2020model,long2017towards, sablayrolles2019white, Yeom2017}).  In this work, we use adversarial analysis to study existing methods. We focus on a fundamental adversarial analysis, called \emph{membership inference} \cite{Shokri2017a}. In this setting, the adversary tries to determine whether a datapoint is part of the training data of a machine learning algorithm. The success rate of the attack shows how much the model would leak about its individual datapoints. 

This approach is not specific to machine learning. \cite{homer2008resolving} demonstrated a successful membership inference attack on aggregated genotype data provided by the US National Institutes of Health and other organizations. This attack was successful despite the NIH witholding public access to their aggregate genome databases \cite{de2008snping}.  With respect to machine learning systems, the UK's information commissioners office explicitly states membership inference as a threat in its guidance on the AI auditing framework \cite{commissionerGuidance2020}. Beyond its practical and legal aspects, this approach is used to measure model information leakage \cite{Shokri2017a}. Privacy-preserving algorithms need to be designed to establish upper bounds on such leakage (notably using differential privacy algorithms, e.g., \cite{abadi2016deep}) 

\paragraph{Our Contributions}\label{sec:contrib} 
Our work is the first to extensively analyze the {\em data} privacy risks that arise from releasing model explanations, which can result in a trade-off between transparency and privacy. This analysis is of great importance, given that model explanations are required to provide transparency about model decisions, and privacy is required to protect sensitive information about the training data. We provide a comprehensive analysis of information leakage on major feature-based model explanations. We analyze both {\em backpropagation-based} model explanations, with an emphasis on gradient-based methods~\cite{Baehrens2009, Klauschen2015, Shrikumar2017, Sliwinski2019, sundararajan2017axiomatic} and \emph{perturbation-based} methods~\cite{Ribeiro2016should, Smilkov2017}. We assume the adversary provides the input query, and obtains the model prediction as well as the explanation of its decision. We analyze if the adversary can trace whether the query was part of the model's training set. 

For gradient-based explanations, we demonstrate \textbf{how and to what extent} backpropagation-based explanations leak information about the training data (Section~\ref{sec:thresholdExperiments}).  Our results indicate that backpropagation-based explanations are a major source of information leakage. We further study the effectiveness of membership inference attacks based on additional backpropagation-based explanations (including Integrated Gradients and LRP). These attacks achieve comparable, albeit weaker, results than attacks using gradient-based explanations.  

We further investigate \textbf{why} this type of model explanation leaks membership information (Section~\ref{sec:factorsForinformationleakage}). Note that the model explanation, in this case, is a vector where each element indicates the influence of each input feature on the model's decision. We demonstrate that the \emph{variance} of a backpropagation-based explanation (i.e, the variance of the influence vector across different features) can help identify the training set members. This link could be partly due to how backpropagation-based training algorithms behave upon convergence. The high variance of an explanation is an indicator for a point being close to a decision boundary, which is more common for datapoints outside the training set. During training the decision boundary is pushed away from the training points. 

This observation links the high variance of the explanation to an uncertain prediction and so indirectly to a higher prediction loss. Points close to the decision boundary have both an uncertain prediction and a high variance in their explanation. This insight helps to explain the leakage. High prediction and explanation variance is a good proxy for a higher prediction loss of the model around an input. This is a very helpful signal to the adversary, as membership inference attacks based on the loss are highly accurate \cite{sablayrolles2019white}: Points with a very high loss tend to be far from the decision boundary and are also more likely to be non-members. 
 
Further, our experiments on synthetic data indicate that the relationship between the variance of an explanation and training data membership is greatly affected by data dimensionality. For low dimensional data, membership is uncorrelated with explanation variance. These datasets are relatively dense. There is less variability for the learned decision boundary and members and non-members are equally likely to be close to it. Interestingly, not even the loss-based attacks are effective in this setting. Increasing the dimensionality of the dataset, and so decreasing its relative density, leads to a better correlation between membership and explanation variance. Finally, when the dimensionality reaches a certain point the correlation decreases again.  This decrease is inline with a decrease in training accuracy for the high dimensional data. Here, the model fails to learn. 

To provide a better analysis of the trade-off between privacy and transparency, we analyze perturbation-based explanations, such as SmoothGrad~\cite{Smilkov2017}.  We show that, as expected, these techniques are more resistant to membership inference attacks (Section~\ref{sec:pertubationExplanations}). We, however, attribute this to the fact that they rely on out-of-distribution samples to generate explanations. These out-of-distribution samples, however, can have undesirable effects on explanation fidelity \cite{slack2020fooling}. So, these methods can achieve privacy at the cost of the quality of model explanations. 

\paragraph{Additional results in supplementary material} 

In the supplementary material, we study another type of model explanation: the example-based method based on influence-functions proposed by \citet{koh2017understanding}. This method provides influential training datapoints as explanations for the decision on a particular point of interest (Appendix~\ref{sec:influence-functions}). This method presents a clear leakage of training data, and is far more vulnerable to membership inference attacks;
in particular, training points are frequently used to explain their own predictions (Appendix~\ref{sec:record}). Hence, for this method, we focus on a more ambitious objective of reconstructing the entire training dataset via \textbf{dataset reconstruction attacks}~\cite{dwork2017exposed}. 

The challenge here is to recover as many training points as possible.
Randomly querying the model does not recover many points. A few peculiar training data records --- especially mislabeled training points at the border of multiple classes --- have a strong influence over most of the input space. Thus, after a few queries, the set of reconstructed data points converges.  We design an algorithm that identifies and constructs regions of the input space where previously recovered points will not be influential (Appendix~\ref{sec:record_analysis}).  This approach avoids rediscovering already revealed instances and improves the attack's coverage.  
We prove a worst-case upper bound on the number of recoverable points and show that our algorithm is optimal in the sense that for worst-case settings, it recovers all discoverable datapoints.

Through empirical evaluation of example-based model explanations on various datasets (Appendix~\ref{sec:dataset-reconstruction}), we show that an attacker \textbf{can reconstruct (almost) the entire dataset for high dimensional data}. For datasets with low dimensionality, we develop another heuristic: by adaptivley querying the previously recovered points, we recover significant parts of the training set. 
Our success is due to the fact that in the data we study, the graph structure induced by the influence function over the training set, tends to have a small number of large strongly connected components, and the attacker is likely to recover at least all points in one of them.

We also study the influence of dataset size on the success of membership inference for example-based explanations. Finally, as unusual points tend to have a larger influence on the training process, we show that the data of \textbf{minorities is at a high risk of being revealed}. 

%% file: PreliminariesAndProblemformulation.tex
\section{Background and Preliminaries}\label{sec:preliminaries} 
We are given a labeled {\em dataset} $\cal X \subseteq \R^n$, with $n$ features and $k$ labels. The labeled dataset is used to train a {\em model} $c$, which maps each {\em datapoint} $\vec x$ in $\R^n$ to a distribution over $k$ {\em labels}, indicating its belief that any given label fits $\vec x$. Black-box models often reveal the label deemed {\em likeliest} to fit the datapoint. The model is defined by a set of {\em parameters} $\theta$ taken from a {\em parameter space} $\Theta$. 
We denote the model as a function of its parameters as $c_\theta$. A model is trained to empirically minimize a {\em loss function} over the training data. The loss function $L:\cal X\times \Theta \to \R$ takes as input the model parameters $\theta$ and a point $\vec x$, and outputs a real-valued loss $L(\vec x,\theta)\in \R$. The objective of a machine-learning algorithm is to identify an {\em empirical loss minimizer} over the parameter space $\Theta$:
\begin{align}
\hat \theta \in \argmin_{\theta \in \Theta} \frac{1}{|\cal X|} \sum_{\vec x \in \cal X}L(\vec x,\theta)\label{eq:total-loss}
\end{align}

\subsection{Model Explanations}\label{sec:transparency-measures}
As their name implies, model explanations explain model decisions on a given {\em point of interest} (POI) $\poi \in \R^n$. 
An explanation $\phi$ takes as input the dataset $\cal X$, labels over $\cal X$ --- given by either the true labels $\ell:\cal X \to [k]$ or by a trained model $c$ --- and a {\em point of interest} $\poi  \in \R^n$. 
Explanation methods sometimes assume access to additional information, such as active access to model queries (e.g. \cite{adler2016auditing,Datta2016,Ribeiro2016should}), a prior over the data distribution \cite{Baehrens2009}, knowledge of the model class (e.g. that the model is a neural network \cite{ancona2017dnns,shrikumar2017deeplift,sundararajan2017axiomatic}, or that we know the source code \cite{datta2017programs,Ribeiro2018}). 
We assume that the explanation function $\phi(\cal X,c,\poi,\cdot)$ is \emph{feature-based} (here the $\cdot$ operator stands for potential additional inputs), and often refer to the explanation of the POI $\poi$ as $\phi(\poi)$, omitting its other inputs when they are clear from context.

The $i$-th coordinate of a feature-based explanation, $\phi_i(\poi)$ is the degree to which the $i$-th feature influences the label assigned to $\poi$. 
Generally speaking, high values of $\phi_i(\poi)$ imply a greater degree of effect; negative values imply an effect for {\em other labels}; a $\phi_i(\poi)$  close to $0$ normally implies that feature $i$ was largely irrelevant. \citet{Ancona2017} provide an overview of feature-based explanations (also called attribution methods). Many feature-based explanation techniques are implemented in the \textsc{innvestigate} library\footnote{\url{https://github.com/albermax/innvestigate}} \cite{Alber2018} which we use in our experiments. Let us briefly review the explanations we analyze in this work.

\subsubsection{Backpropagation-Based Explanations} Backpropagation-based methods rely on a small number of backpropagations through a model to attribute influence from the prediction back to each feature. The canonical example of this type of explanation is the gradient with respect to the input features~\cite{Simonyan2013a}, we focus our analysis on this explanation. Other backpropagation-based explanations have been proposed \cite{Baehrens2009, Klauschen2015, Shrikumar2017, Sliwinski2019, Smilkov2017, sundararajan2017axiomatic}.

\paragraph{Gradients} 
\citet{Simonyan2013a} introduce gradient-based explanations to visualize image classification models, i.e. $\phi_{i}(\poi) = \frac{\partial c}{\partial x_i}(\poi)$. The authors utilize the absolute value of the gradient, i.e. $\left|\frac{\partial c}{\partial x_i}(\poi)\right|$; however, outside image classification, it is reasonable to consider negative values, as we do in this work.
We denote gradient-based explanations as $\phi_{\GRAD}$. \citet{Shrikumar2017} propose setting $\phi_i(\poi) = y_i \times \frac{\partial c}{\partial x_i}(\poi)$ 
as a method to enhance numerical explanations. 
Note that since an adversary would have access to $\poi$, releasing its Hadamard product with $\phi_{\GRAD}(\poi)$ is equivalent to releasing $\phi_{\GRAD}(\poi)$. 

\paragraph{Integrated Gradients} 
\newcommand{\BL}{\mathit{BL}}

\citet{sundararajan2017axiomatic} argue that instead of focusing on the gradient it is better to compute the average gradient on a linear path to a baseline $\vec x_{\BL}$ (often $\vec x_{\BL} = \vec 0$). This approach satisfies three desirable axioms: sensitivity, implementation invariance and a form of completeness. 
Sensitivity means that given a point $\vec x \in \cal X$ such that $x_i \ne x_{\BL,i}$ and $c(\vec x) \ne c(\vec x_{\BL})$, then $\phi_i(\vec x) \ne 0$; completeness means that $\sum_{i =1}^n \phi_i(\vec x) = c(\vec x) - c(\vec x_{\BL})$. 
Mathematically the explanation can be formulated as
\[
\phi_{\INTGRAD}(\vec{x})_i \triangleq (x_i - \vec x_{\BL,i})  \cdot  \left. \int_{\alpha = 0}^1 \frac{\partial c(\vec{x}^\alpha)}{\partial \vec{x}^\alpha_i} \right|_{\vec{x}^\alpha = \vec{x}+\alpha (\vec{x} - \vec x_{\BL})}.
\]  

\paragraph{Guided Backpropagation}
Guided Backpropagation \cite{Springenberg2014} is a method specifically designed for networks with ReLu activations. It is a modified version of the gradient where during backpropagation only paths are taken into account that have positive weights and positive ReLu activations. Hence, it only considers positive evidence for a specific prediction. While being designed for ReLu activations it can also be used for networks with other activations.

\paragraph{Layer-wise Relevance Propagation (LRP)} \citet{Klauschen2015} use backpropagation to map \emph{relevance} back from the output layer to the input features. 
LRP defines the relevance in the last layer as the output itself and in each previous layer the relevance is redistributed according to the weighted contribution of the neurons in the previous layer to the neurons in the current layer.  The final attributions for the input $\vec{x}$ are defined  as the attributions of the input layer. We refer to this explanation as $\phi_{\LRP}(\vec{x})$.

\subsubsection{Perturbation-Based Explanations} 
Perturbation-based methods query the to-be-explained model on many perturbed inputs. They either treat the model as a black-box~\cite{Datta2015influence, Ribeiro2016should}, need predictions for counterfactuals~\cite{Datta2015influence}, or `smooth' the explanation~\cite{Smilkov2017}. They can be seen as local linear approximations of a model. 

\paragraph{SmoothGrad} We focus our analysis on SmoothGrad~\cite{Smilkov2017}, which generates multiple samples by adding Gaussian noise to the input and releases the averaged gradient of these samples. Formally for some $k\in \mathbb{N}$,
\[
\phi_{\SMOOTH}(\vec x) = \frac{1}{k} \sum_k \nabla c(\vec x + \mathcal{N}(0, \sigma)),
\]
where $\mathcal{N}$ is the normal distribution and $\sigma $ is a hyperparameter.

\paragraph{LIME} The LIME (Local Interpretable Model-agnostic Explanations) method \cite{Ribeiro2016should} creates a local approximation of the model via sampling. Formally it solves the following optimization problem:
\[
\phi_{\LIME}(\vec x) = \argmin_{g \in G} \mathcal{L}( g,c,\pi_{\vec x} ) +\Omega(g),
\]
where $G$ is a set of simple functions, which are used as explanations, $\mathcal{L}$ measures the approximation quality by $g$ of $c$ in the neighborhood of $\vec x$ (measured by $\pi_{\vec x}$ ) and $\Omega$ regularizes the complexity of $g$. 
While the LIME framework allows for an arbitrary local approximation in practice most commonly used is a linear approximation with Ridge regularization. 

%% file: MembershipInferenceAttacks.tex
\subsection{Membership Inference Attacks}\label{sec:thresholdDescription}
We assume the attacker has gained possession of a set of datapoints $S \subset R^n$, and would like to know which ones are members of the training data. The goal of a membership inference attack is to create a function that accurately predicts whether a point $\vec x \in S$ belongs to the training set of $c$. The attacker has a prior belief how many of the points in $S$ were used for training. 
In this work we ensure that half the members of $S$ are members of the training set (this is known to the attacker), thus random guessing always has an accuracy of 50\%, and is the threshold to beat.

Models tend to have lower loss on members of the training set. Several works have exploited this fact to define simple loss-based attacks \cite{long2017towards, sablayrolles2019white, Yeom2017}. The idea is to define a threshold $\tau$: an input $\vec x$ with a loss $L(\vec x,\theta)$  lower than $\tau$ is considered a member; an input with a loss higher is considered a non-member.
\[
\text{Membership}_{\Loss,\tau}(\vec x) = \begin{cases}
\True &\text{if } L(\vec x,\theta) \leq \tau \\
\False &\text{otherwise }  \\
\end{cases}
\]
\citet{sablayrolles2019white} show that this attack is optimal given an optimal threshold $\tau_\text{opt}$, under some assumptions.  
However, this attack is infeasible when the attacker does not have access to the true labels or the model's loss function.

Hence, we propose to generalize threshold-based attacks to allow different sources of information. For this we use the notion of variance for a given vector $\vec{v} \in \R^n$:
\begin{align*}
\Var(\vec{v}) \triangleq &\sum_{i=1}^n (v_i-\mu_{\vec{v}})^2 &\text{where } \mu_{\vec v} = \frac{1}{n}\sum_{i=1}^n v_i
\end{align*}
Explicitly, we  consider \begin{enumerate*}[label=(\roman*)] \item a threshold on the prediction variance and \item a threshold on the explanation variance \end{enumerate*}. The target model usually provides access to both these types of information. Note, however, a target model might \emph{only release the predicted label and an explanation}, making only explanation-based attacks feasible.

Our explanation-based threshold attacks work in a similar manner to other threshold-based attack models: $\poi$ is considered a member iff $\Var(\phi(\poi))\le \tau$. 

\begin{align*}
\text{Membership}_{\Prediction,\tau}(\vec x) &= \begin{cases}
\True &\text{if } \Var(c_\theta(\vec x)) \geq \tau \\
\False &\text{otherwise }  \\
\end{cases}\\
\text{Membership}_{\Explanation,\tau}(\vec x) &= \begin{cases}
\True &\text{if } \Var(\phi(\vec x)) \leq \tau \\
\False &\text{otherwise }  \\
\end{cases}
\end{align*}

Intuitively, if the model has a very low loss then its prediction vector will be dominated by the true label. These vectors have higher variance than vectors where the prediction is equally distributed among many labels (indicating model uncertainty). 
This inference attack breaks in cases where the loss is very high because the model is decisive but wrong. However, as we demonstrate below, this approach offers a fairly accurate attack model for domains where loss-based attacks are effective. Hence, attacks using prediction variance alone still constitute a serious threat.
The threshold attack based on explanation variance are similarly motivated. When the model is certain about a prediction, it is also unlikely to change it due to a small local perturbation. Therefore, the influence and attribution of each feature are low, leading to a smaller explanation variance. 
For points closer to the decision boundary, changing a feature affects the prediction more strongly, leading to higher explanation variance. The loss minimization during training ``pushes'' points away from the decision boundary. In particular, models using $\tanh$, sigmoid, or softmax activation functions tend to have steeper gradients in the areas where the output changes. Training points generally don't fall into these areas.\footnote{The high variance described here results from higher absolute values, in fact instead of the variance an attacker could use the 1-norm. In our experiments, there was no difference between using 1-norm and using variance; we decided to use variance to be more consistent with the prediction based attacks.}
The crucial part for all threshold-based attacks is obtaining the threshold $\tau$. We consider two scenarios:
\setdefaultleftmargin{0.5cm}{}{}{}{}{}
\begin{enumerate}
	\item \textbf{Optimal threshold} For a given set of members and non-members there is a threshold $\tau_\text{opt}$ that achieves the highest possible prediction accuracy for the attacker. 
	This threshold can easily be obtained when datapoint membership is known. Hence, rather than being an actually feasible attack, using  $\tau_\text{opt}$ helps estimating the worst case privacy leakage.
	\item \textbf{Reference/Shadow model(s)} This setting assumes that the attacker has access to some labeled data from the target distribution. The attacker trains $s$ models on that data and calculates the threshold for these reference (or shadow) models. In line with Kerckhoffs’s principle \cite{petitcolas2011kerckhoffs} we assume that the attacker has access to the training hyper parameters and model architecture. This attack becomes increasingly resource intensive as $s$ grows. For our experiments we choose $s \in \{1,3\}$. This is a practically feasible attack if the attacker has access to similar data sources. 
\end{enumerate}

%% file: ExperimentsForGradientBasedExplanations.tex
\section{Privacy Analysis of Backpropagation-Based Explanations}\label{sec:thresholdExperiments} 
In this section we describe and evaluate our membership inference attack on gradient-based explanation methods. 
  We use the \purchase and \texas datasets in \cite{Nasr2018}; we also test \cifarTen and \cifarHundred \cite{sablayrolles2019white}, the \adult dataset \cite{Dua2017} as well as the \hospital dataset \cite{Strack2014}. 
  The last two datasets are the only binary classification tasks considered. 
  Where possible, we use the same training parameters and target architectures as the original papers (see Table~\ref{tab:targetDataset} for an overview of the datasets). We study four types of information the attacker could use: loss, prediction variance, gradient variance and the SmoothGrad variance. 
  
   \begin{table}[h!]
  	\centering{	
  		\caption{Overview of the target datasets for membership inference}
  		\label{tab:targetDataset}
  		\begin{tabular}{l r r l r }
  			\toprule
  			Name & Points & Features & Type & \# Classes \\
  			\midrule 
  			\purchase  & 197,324 & 600    & Binary & 100 \\
  			\texas     & 67,330  & 6,170  & Binary & 100 \\ 
  			\cifarHundred & 60,000   & 3,072  &Image  & 100 \\ 
  			\cifarTen  & 60,000   & 3,072  & Image & 10 \\ 
  			\hospital  & 101,766 & 127    & Mixed  & 2\\ 
  			\adult     & 48,842  & 24     & Mixed  & 2\\ 
  			\bottomrule
  	\end{tabular}}
  \end{table}

   \begin{table}[h!]
   	\setlength{\tabcolsep}{3pt}
	\centering{	
		\caption{The average training and testing accuracies of the target models.}
		\label{tab:targetAccurcaies}
		\begin{tabular}{l r r r r r r }
			\toprule
			 			& \purchase & \texas & CIFAR  	& CIFAR 	& \hospital  	& \adult  \\
			 			&  &  & -100  	& -10 	&   	&   \\
			\midrule 
			Train  	& 1.00		& 0.98		&	0.97		& 0.93			& 0.64				& 0.85 \\
			Test   	& 0.75		& 0.52		&	0.29		& 0.53			& 0.61				& 0.85 \\
			\bottomrule

	\end{tabular}
}
\end{table}

  \subsection {General setup}
 For all datasets, we first create one big dataset by merging the original training and test dataset, to have a large set of points for sampling. Then, we randomly sample four smaller datasets that are not overlapping. We use the smaller sets to train and test four target models and conduct four attacks. In each instance, the other three models can respectively be used as shadow models. We repeat this process 25 times, producing a total of 100 attacks for each original dataset. Each small dataset is split 50/50 into a training set and testing set. Given the small dataset,  the attacker has an a priori belief that 50\% of the points are members of the training set, which is the common setting for this type of attack \cite{Shokri2017}.
  \subsection{Target datasets and architectures}
  
  The overview of the datasets is provided in Table~\ref{tab:targetDataset} and an overview of the target models accuracies in Table~\ref{tab:targetAccurcaies}.
  
  \subsubsection{Purchase dataset}
  The dataset originated from the ``Acquire Valued Shoppers Challenge'' on Kaggle\footnote{\url{https://www.kaggle.com/c/acquire-valued-shoppers-challenge/data}}. 
The goal of the challenge was to use customer shopping history to predict shopper responses to offers and discounts. 
For the original membership inference attack, \citet{Shokri2017} create a simplified and processed dataset, which we use as well. Each of the 197,324 records corresponds to a customer. The dataset has 600 binary features representing customer shopping behavior. The prediction task is to assign customers to one of 100 given groups (the labels). 
This learning task is rather challenging, as it is a multi-class learning problem with a large number of labels; moreover, due to the relatively high dimension of the label space, allowing an attacker access to the prediction vector --- as is the case in \cite{Shokri2017} --- represents significant access to information.
We sub-sampled smaller datasets of 20,000 points i.e. 10,000 training and testing points for each model. We use the same architecture as \cite{Nasr2018}, namely a four-layer fully connected neural network with $\tanh$ activations. The layer sizes are [1024, 512, 256, 100]. We trained the model of 50 epochs using the Adagrad optimizer with a learning rate of 0.01 and a learning rate decay of 1e-7.

  \subsubsection{Texas hospital stays}
  The Texas Department of State Health Services released hospital discharge data public use files spanning from 2006 to 2009.\footnote{\url{https://www.dshs.texas.gov/THCIC/Hospitals/Download.shtm}} The data is about inpatient status
at various health facilities. There are four different groups of attributes in each record: general information (e.g., hospital id, length of stay, gender, age, race), the diagnosis, the procedures the patient underwent, and the external causes of injury. The goal of the classification model is to predict the patient’s primary procedures based on the remaining attributes (excluding the secondary procedures). The dataset is filtered to include only the 100 most common procedures.
The features are transformed to be binary resulting in 6,170 features and 67,330 records. 
We sub-sampled smaller datasets of 20,000 points i.e. 10,000 training and testing points for each model. As the dataset has only 67,330 points we allowed resampling of points. We use the same architecture as \cite{Nasr2018}, namely a five-layer fully connected neural network with $\tanh$ activations. The layer sizes are [2048, 1024, 512, 256, 100]. We trained the model of 50 epochs using the Adagrad optimizer with a learning rate of 0.01 and a learning rate decay of 1e-7. 
  \subsubsection{\cifarTen and \cifarHundred}
\cifarTen  and \cifarHundred are well-known benchmark datasets for image classification \cite{krizhevsky2009learning}. They consists of 10 (100) classes of $32 \times 32 \times 3$ color images, with 6,000 (600) images per class. The datasets are usually split in 50,000 training and 10,000 test images. 
For \cifarTen, we use a small convolutional network with the same architecture as in \cite{Shokri2017, sablayrolles2019white}, it has two convolutional layers with max-pooling, and two dense layers, all with Tanh activations. We train the model for 50 epochs with a learning rate of 0.001 and the Adam optimizer. Each dataset has 30,000 points (i.e. 15,000 for training). Hence, we only have enough points to train one shadow model per target model.   
For \cifarHundred, we use a version of Alexnet \cite{krizhevsky2012imagenet}, it has five convolutional layers with max-pooling, and to dense layers, all with ReLu activations. We train the model for 100 epochs with a learning rate of 0.0001 and the Adam optimizer. Each dataset has 60,000 points (i.e. 30,000 for training). Hence, we don't have enough points to train shadow models. However, with a smaller training set, there would be too few points of each class to allow for training.
  \subsubsection{UCI Adult (Census income)}  This dataset is extracted from the 1994 US Census database \cite{Dua2017}. It contains 48,842 datapoints. It is based on 14 features (e.g., age, workclass, education). The goal is to predict if the yearly income of a person is above 50,000 \$. We transform the categorical features into binary form resulting in 104 features.
  We sub-sampled smaller datasets of 5,000 points i.e. 2,500 training and testing points for each model. For the architecture, we use a five-layer fully-connected neural network with Tanh activations. The layer sizes are [20, 20, 20, 20, 2]. We trained the model of 20 epochs using the Adagrad optimizer with a learning rate of 0.001 and a learning rate decay of 1e-7. 
  \subsubsection{Diabetic Hospital} The dataset contains data on diabetic
  patients from 130 US hospitals and integrated delivery networks
  \cite{Strack2014}. We use the modified version described in \cite{koh2017understanding}
  where each patient has 127 features which are demographic (e.g.
  gender, race, age), administrative (e.g., length of stay), and medical
  (e.g., test results); the prediction task is readmission within 30 days (binary).
  The dataset contains 101,766 records from which we sub-sample
  balanced (equal numbers of patients from each class) datasets of size 10,000. Since the original dataset is heavily biased towards one class, we don't have enough points to train shadow models.
  As architecture, we use a four-layer fully connected neural network with Tanh activations. The layer sizes are [1024, 512, 256, 100]. We trained the model for 1,000 epochs using the Adagrad optimizer with a learning rate of 0.001 and a learning rate decay of 1e-6.

\begin{figure*}
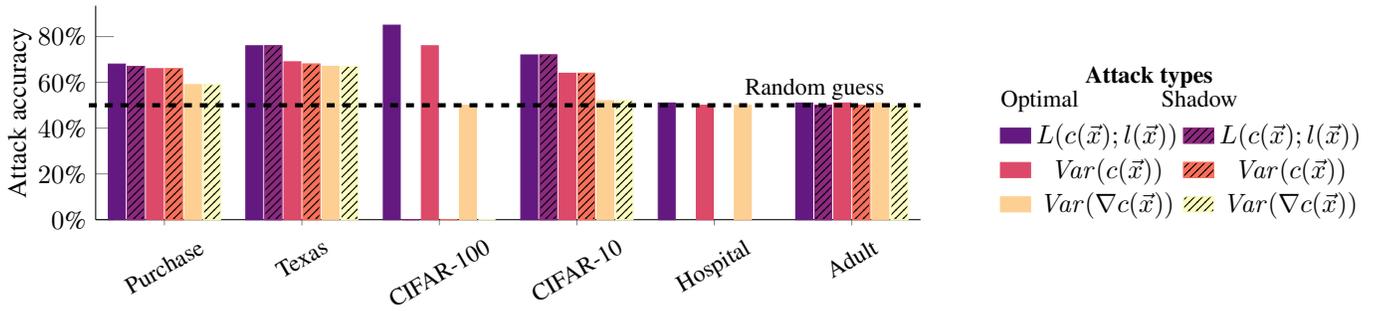

	\centering
	\if\compileFigures1
	\input{figure_scripts/fig_threshold_attacks}
	\else
	\includegraphics[]{fig/\filename-figure\thefiguerNumber.pdf}
	\stepcounter{figuerNumber}
	\fi
	\caption{Results for the threshold-based attacks using different attack information sources. The \textsc{Optimal} attack uses the optimal threshold; the \textsc{Shadow} trains a shadow model on data from the same distribution, and uses an optimal threshold for the shadow model. Using three such models results in nearly optimal attack accuracy.}
	\label{fig:thresholdattacks}
\end{figure*}
\subsection{Evaluation of main experiment}
\paragraph{Explanation-based attacks}
As can be seen in Figure~\ref{fig:thresholdattacks}, gradient-based attacks (as well as other backpropagation-based methods, as further discussed in Section~\ref{app:backpropbased}) on the \purchase and \texas datasets were successful. This result is a \emph{clear proof of concept}, that model explanations are exploitable for membership inference. 
However, the attacks were ineffective for the image datasets; gradient variance fluctuates wildly between individual images, making it challenging to infer membership based on explanation variance. 
\paragraph{Loss-based and predictions-based attacks}
When loss-based attacks are successful, attacks using prediction variance are nearly as successful. These results demonstrate that it is not essential to assume that the attacker knows the true label of the point of interest.  
\paragraph{Types of datasets}
The dataset type (and model architecture) greatly influences attack success. For both binary datasets (\texas and \purchase), all sources of information pose a threat. 
On the other hand, for the very low dimensional \hospital and \adult datasets, none of the attacks outperform random guessing. This lack of performance may be because the target models do not overfit to the training data (see Table~\ref{tab:targetAccurcaies}), which generally limits its vulnerability to adversarial attacks \cite{Yeom2017}.
\paragraph{Optimal threshold vs. shadow models}
Shadow model-based attacks compare well to the optimal attack, with attacks based on three shadow models performing nearly at an optimal level; this is in line with results for loss-based attacks \cite{sablayrolles2019white}.  

\paragraph{Considering the entire explanation vector}\label{app:learning}
In the attacks above, we used only the variance of the explanations.
Intuitively, when the model is certain about a prediction because it is for a training point, it is unlikely to change the prediction with small local perturbation. Hence, the influence (and attribution) of each feature is low. It has a smaller explanation variance. 
For points closer to the decision boundary, changing a feature affects the prediction more strongly.  The variance of the explanation for those points should be higher. The loss minimization during training tries to ``push'' points away from the decision boundary. Especially, models using $\tanh$, sigmoid, or softmax activations have steep gradients in the areas where the output changes. Training points generally don't fall into these areas.\footnote{The high variance described here results from higher absolute values. Instead of the variance, an attacker could use the 1-norm. In our experiments, there was no difference between using 1-norm and using the variance. We decided to use variance to be more consistent with the attacks based on the prediction threshold.}
Hence, explanation variance is a sufficient parameter for deploying a successful attack. To further validate this claim, we conduct an alternative attack using the entire explanation vector as input. 
\input{appendices/Appendix_Learning}

\subsection{Combining different information sources}\label{sec:differentSources}
The learning attacks described in the previous paragraph allow for a combination of different information sources. For example, an attacker can train an attack network using both the prediction and the explanation as input.  Experiments on combining the three information sources (explanation, prediction, and loss) lead to outcomes identical to the strongest used information source. Especially if the loss is available to an attacker, we could not find evidence that either the prediction vector or an explanation reveals additional information.

\subsection{Results for other backpropagation-based explanations}\label{app:backpropbased}
Besides the gradient, several other explanation methods based on backpropagation have been proposed. We conducted the attack described in Section~\ref{sec:thresholdDescription} replacing the gradient with some other popular of these explanation methods. The techniques are all implemented in the \textsc{innvestigate} library\footnote{\url{https://github.com/albermax/innvestigate}} \cite{Alber2018}. An in-depth discussion of some of these measures, and the relations between them, can also be found in \cite{Ancona2017}.
As can be seen in Figure~\ref{fig:backpropattacks} on the \purchase, \texas, and \cifarTen datasets, the results for other backpropagation based methods are relatively similar to the attack based on the gradient. Integrated gradients performing most similar to the gradient. For \adult, \hospital and \cifarHundred small-scale experiments indicated that this type of attack would not be successful for these explanations as well, we omitted the datasets from further analysis.   

\begin{figure}
	\centering
	\if\compileFigures1
	\input{figure_scripts/fig_backpropattacks}
	\else
	\includegraphics[]{fig/\filename-figure\thefiguerNumber.pdf}
	\stepcounter{figuerNumber}
	\fi
	\caption{Results for the threshold-based attacks using different backpropagation-based explanations as sources of information for the attacker.}
	\label{fig:backpropattacks}
\end{figure}

\section{Analysis of factors of information leakage}\label{sec:factorsForinformationleakage}
In this section, we provide further going analysis to validate our hypothesis and broaden understanding.

\input{appendices/Appendix_Dimensionality}

\subsection{Using individual thresholds}
\citet{sablayrolles2019white} proposed an attack where the attacker obtains a specific threshold for each point (instead of one per model). However, to be able to obtain such a threshold, the attacker would need to train shadow models including the point of interest. This situation would require knowledge of the true label of the point. This conflicts with the assumption that when using explanations (or predictions) for the attack the attacker does not have access to these true labels. Furthermore, \citet{sablayrolles2019white} results suggest that this attack only very mildly improves performance. 

\subsection{Influence of overfitting}

\citet{Yeom2017} show that {\em overfitting significantly influences the accuracy of membership inference attacks}. To test the effect of overfitting, we vary the number of iterations of training achieving different accuracies. 
In line with previous findings for loss-based attacks, our threshold-based attacks using explanations and predictions work better on overfitted models; see Figure~\ref{fig:overfitting}. 

\begin{figure}
	\centering
	\if\compileFigures1
	\input{figure_scripts/fig_overfitting}
	\else
	\includegraphics[]{fig/\filename-figure\thefiguerNumber.pdf}
	\stepcounter{figuerNumber}
	\fi
	\caption{The attack accuracy of the attacker increases with increasing number of epochs.}
	\label{fig:overfitting}
\end{figure}

%% file: appendices/Appendix_Learning.tex
The fundamental idea is to cast membership inference as a {\em learning problem}: the attacker trains an \emph{attack model} that, given the output of a \emph{target model} can predict whether or not the point $\vec x$ was used during the training phase of $c$. 
The main drawback of this approach is that it assumes that the attacker has partial knowledge of the initial training set to train the attack model. \citet{Shokri2017} circumvent this by training {\em shadow models} (models that mimic the behavior of $c$ on the data) and demonstrate that comparable results are obtainable even when the attacker does not have access to parts of the initial training set. As we compare the results to the optimal threshold, it is appropriate to compare with a model that is trained using parts of the actual dataset. This setting allows for a stronger attack.

The specific attack architecture,  we use in this section, is a neural network inspired by the architecture of \citet{Shokri2017}. The network has fully connected layers of sizes $[r, 1024, 512, 64, 256, 64, 1]$, where $r$ is the dimension of the respective explanation vector.
We use ReLu activations between layers and initialize weights in a manner similar to \citet{Shokri2017} to ensure a valid comparison between the methods. We trained the attack model for 15 epochs using the Adagrad optimizer with a learning rate 0.01 of and a learning rate decay of 1e-7. As data for the attacker, we used 20,000 explanations generated by the target 10,000 each for members and non-members. The training testing split for the attacker was 0.7 to 0.3. We repeated the experiment 10 times. We omitted \cifarHundred for computational reasons.

As can be seen in Figure~\ref{fig:thresholdVsentire}, attacks based on the entire explanation perform slightly better than attacks based only on the variance. However, they are qualitatively the same and still perform very poorly for \cifarTen, \adult, and \hospital. 

\begin{figure}
	\centering
	\if\compileFigures1
	\input{figure_scripts/fig_threshold_vs_entire}
	\else
	\includegraphics[]{fig/\filename-figure\thefiguerNumber.pdf}
	\stepcounter{figuerNumber}
	\fi
	\caption{A comparison between attacks using only the variance of the gradient and attacks using the entire gradient explanation as input.}
	\label{fig:thresholdVsentire}
\end{figure}

%% file: appendices/Appendix_Dimensionality.tex
\subsection{The Influence of the Input Dimension}\label{app:dimensionality}
The experiments in Section~\ref{sec:thresholdExperiments} indicate that $Var(\nabla c(\vec x))$, and $||\nabla c(\vec x)||_1$ predict training set membership. 
In other words, high absolute gradient values at a point $\vec x$ signal that $\vec x$ is {\em not} part of the training data: the classifier is uncertain about the label of $\vec x$, paving the way towards a potential attack.
Let us next study this phenomenon on synthetic datasets, and the extent to which an adversary can exploit model gradient information in order to conduct membership inference attacks.  
The use of artificially generated datasets offers us control over the problem complexity, and helps identify important facets of information leaks.

To generate datasets, we use the Sklearn python library.\footnote{the \texttt{make\_classification} function \url{https://scikit-learn.org/stable/modules/generated/sklearn.datasets.make\_classification.html}} 

For $n$ features, the function creates an $n$-dimensional hypercube, picks a vertex from the hypercube as center of each class, and samples points normally distributed around the centers. In our experiments, the number of classes is either 2 or 100 while the number of features is between 1 to 10,000 in the following steps,
\begin{align*}
n \in \{&1, 2, 5, 10, 14, 20, 50, 100, 127, 200, 500, 600, \\ 
& 1000, 2000, 3072, 5000, 6000, 10000\}.
\end{align*} 
For each experiment, we sample 20,000 points and split them evenly into training and test set. We train a fully connected neural network with two hidden layers with fifty nodes each, the $\tanh$ activation function between the layers, and softmax as the final activation. 
The network is trained using Adagrad with learning rate of 0.01 and learning rate decay of $1\euler-7$ for 100 epochs.

Increasing the number of features does not increase the complexity of the learning problem as long as the number of classes is fixed.
However, the dimensionality of the hyper-plane increases, making its description more complex. Furthermore, for a fixed sample size, the dataset becomes increasingly sparse, potentially increasing the number of points close to a decision boundary. Increasing the number of classes increases the complexity of the learning problem. 

Figure~\ref{fig:correlation_synthetic} shows the correlation between $||\nabla c(\vec x)||_1$ and training membership. 
For datasets with a small number of features ($\leq 10^2$) there is almost no correlation. This corresponds to the failure of the attack for \adult and the \hospital dataset. 
When the number of features is in the range ($10^3 \sim 10^4$) there is a correlation, which starts to decrease when the data dimension is further increased. 
The number of classes seems to play only a minor role; however, a closer look at training and test accuracy reveals that the actual behavior is quite different. 
For two classes and a small number of features training and testing accuracy are both high (almost 100\%), around $n = 10^2$ the testing accuracy starts to drop (the model overfits)  and at $n = 10^3$ the training accuracy starts to drop as well reducing the overfitting. 
For 100 classes the testing accuracy is always low and only between  $10^3 \leq n \leq 10^4$  the training accuracy is high, leading to overfitting, just on a lower  level. 
We also conduct experiments with networks of smaller/larger capacity, which have qualitatively similar behavior. 
However, the interval of $n$ in which correlation exists and the amount of correlation varies (see Figure~\ref{fig:cap} in Appendix~\ref{app:capacity}).  

\begin{figure}
	\centering
	\if\compileFigures1	
	\input{figure_scripts/fig_synthetic_datasets}
	\else
	\includegraphics[]{fig/\filename-figure\thefiguerNumber.pdf}
	\stepcounter{figuerNumber}
	\fi
	\caption{\small The correlation between $||\nabla c(\vec x)||_1$ and training membership for synthetic datasets for increasing number of  features $n$ and different number of classes $k \in \{ 2, 100 \}$}
	\label{fig:correlation_synthetic}
\end{figure}

%% file: DiscussionPerturbationBasedExplanations.tex
\section{Privacy Analysis of Perturbation-Based Explanations}\label{sec:pertubationExplanations}
Neither the threshold-based attacks described in Section~\ref{sec:thresholdDescription} nor the learning-based attacks in Section~\ref{app:learning} outperform random guessing when given access to the SmoothGrad \cite{Smilkov2017}. Given that SmoothGrad is using sampling rather than a few backpropagations, it is inherently different from the other explanations we considered so far. We discuss the differences in this section.

\subsection{Attacks using LIME  explanation}\label{app:lime}
As a second perturbation-based method, we looked at the popular explanation method LIME \cite{ribeiro2016programs}. The type of attack is the same as described in Section~\ref{sec:thresholdDescription}. We use an optimal threshold based on the variance of the explanation. However, the calculation of LIME explanations takes considerably longer than the computation of other methods we considered. Every single instance computes for a few seconds. Running experiments with 10,000 or more explanations would take weeks to months. To save time and energy, we restricted the analysis of the information-leakage of LIME to smaller-scale experiments where the models train on 1,000 points, and the attacks run on 2,000 points each (1,000 members and 1,000 non-members). We also repeated each experiment only 20 times instead of 100 as for the others. Furthermore, given that the experiments for the other explanations indicated that only for \purchase and \texas the attack was likely to be successful, we restricted our experiments to these two datasets. Figure~\ref{fig:lime} shows the results for these attacks. To ensure that it is not the different setting that determines the outcome, we also rerun the attacks for the gradient and SmoothGrad explanations, as well as the attack using the prediction variance in this new setting. Neither LIME nor SmoothGrad outperforms random guessing. For the \purchase dataset, however, the attack using the gradient variance fails as well. As a final interesting observation, which we are unable to explain at the moment: For the \texas dataset, the gradient-based attack performs better than on the larger dataset (shown in Figure~\ref{fig:thresholdattacks}) it even outperforms the attack based on the prediction in this specific setting. Something we want to explore further in future works.
\begin{figure}
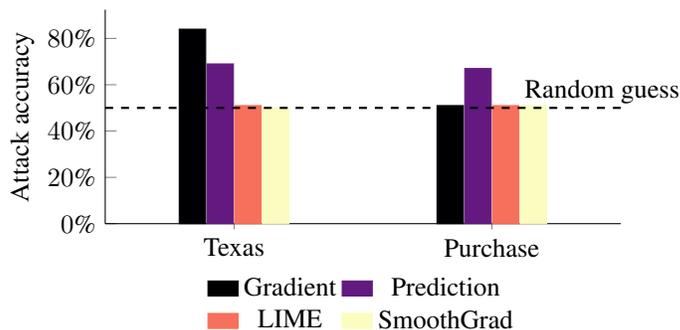

	\centering
	\if\compileFigures1
	\input{figure_scripts/fig_lime}
	\else
	\includegraphics[]{fig/\filename-figure\thefiguerNumber.pdf}
	\stepcounter{figuerNumber}
	\fi
	\caption{Attacks using \LIME or SmoothGrad do not outperform random guessing in any of our experiments.}
	\label{fig:lime}
\end{figure}

\subsection{Analysis}
While it is entirely possible that perturbation-based methods are vulnerable to membership inference, we conjecture that this is not the case. 
This conjecture is due to an interesting connection between perturbation-based model explanations and the {\em data-manifold hypothesis} \cite{fefferman2016testing}.
The data-manifold hypothesis states that ``data tend to lie near a low dimensional manifold'' \cite[p. 984]{fefferman2016testing}. Many works support this hypothesis \cite{belkin2003laplacian, brand2003charting,narayanan2010sample}, and use it to explain the pervasiveness of adversarial examples \cite{gilmer2018adversarial}. To the best of our knowledge, little is known on how models generally perform outside of the data manifold. 
In fact, it is not even clear how one would measure performance of a model on points outside of the training data distribution: they do not have any natural labels. 
Research on creating more robust models aims at decreasing model sensitivity to small perturbations, including those yielding points outside of the manifold. However, robustness results in vulnerability to membership inference \cite{song2019privacy}.   
Perturbation-based explanation methods have been criticized for not following the distribution of the training data and violating the manifold hypothesis \cite{kumar2020problems, sundararajan2019shapley}. \citet{slack2020fooling} demonstrate how a malicious actor can differentiate normal queries to a model from queries generated by LIME and QII, and so make a biased model appear fair during an audit. 
Indeed, the resilience of perturbation-based explanations to membership inference attacks may very well stem from the fact that query points that the model is not trained over, and for which model behavior is completely unspecified. 
One can argue that the fact that these explanations do not convey membership information is a major {\em flaw} of this type of explanations. Given that the results in the previous section indicate that for many training points the model heavily overfits --- to the extent that it effectively ``memorizes'' labels --- an explanation should reflect that.

%% file: BroaderImpact.tex
\section{Broader Impact}

AI governance frameworks call for transparency and privacy for machine learning systems.\footnote{See, for example, the white paper by the European Commission on Artificial Intelligence -- A European approach to excellence and trust:  \url{https://ec.europa.eu/info/sites/info/files/commission-white-paper-artificial-intelligence-feb2020_en.pdf}}
Our work investigates the potential negative impact of explaining machine learning models, in particular, it shows that offering model explanations may come at the cost of user privacy. 
The demand for automated model explanations led to the emergence of model explanation suites and startups. However, none of the currently offered model explanation technologies offer any provable privacy guarantees. This work has, to an extent, arisen from discussion with colleagues in industry and AI governance; both expressed a great deal of interest in the potential impact of our work on the ongoing debate over model explainability and its potential effects on user privacy. 

One of the more immediate risks is that a real-world malicious entity uses our work as the stepping stone towards an attack on a deployed ML system. While our work is still preliminary, this is certainly a potential risk. Granted, our work is still at the proof-of-concept level, and several practical hurdles must be overcome in order to make it into a fully-fledged deployed model, but nevertheless the risk exists. In addition, to the best of our knowledge, model explanation toolkits have not been applied commercially on high-stakes data. Once such explanation systems are deployed on high-stakes data (e.g., for explaining patient health records or financial transactions), a formal exploration of their privacy risks (as is offered in this work) is necessary.

Another potential impact --- which is, in the authors' opinion, more important --- is that our work raises the question whether there is an {\em inevitable} conflict between explaining ML models --- the celebrated ``right to explanation'' --- and preserving user privacy. This tradeoff needs to be communicated beyond the ML research community, to legal scholars and policymakers. Furthermore, some results on example-based explanations suggest that the explainability/privacy conflict might disparately impact minority groups: their data is either likelier to be revealed, else they will receive low quality explanations. We do not wish to make a moral stand in this work: explainability, privacy and fairness are all noble goals that we should aspire to achieve. Ultimately, it is our responsibility to explain the capabilities --- and limitations --- of technologies for maintaining a fair and transparent AI ecosystem to those who design policies that govern them, and to various stakeholders.
Indeed, this research paper is part of a greater research agenda on transparency and privacy in AI, and the authors have initiated several discussions with researchers working on AI governance. 
The tradeoff between privacy and explainability is not new to the legal landscape \cite{banisar2011}; we are in fact optimistic about finding model explanation methods that do not violate user privacy, though this will likely come at a cost to explanation quality.

Finally, we hope that this work sheds further light on what constitutes a {\em good} model explanation. The recent wave of research on model explanations has been recently criticized for lacking a focus on actual usability \cite{kaur2020interpreting}, and for being far from what humans would consider helpful. It is challenging to mathematically capture human perceptions of explanation quality. However, our privacy perspective does shed some light on when explanations are not useful: explanations that offer no information on the model are likely to be less human usable (note that from our privacy perspective, we do not want private user information to be revealed, but revealing some model information is acceptable).

%% file: Conclusion.tex
\section{Related Work and  Conclusions}\label{sec:conclusion}
\citet{Milli2018} show that gradient-based explanations can be used to reconstruct the underlying model; in recent work, a similar reconstruction is demonstrated based on counterfactual explanations \cite{aivodji2020model}  this serves as additional evidence of the vulnerability of transparency reports. However, copying the behavior of a model is different from the inference of its training data. While the former is unavoidable, as long the model is accessible, the latter is more likely an undesired side effect of current methods.
There exists some work on the defense against privacy leakage in advanced machine learning models.  \citet{abadi2016deep} and \citet{Papernot2018a} have designed frameworks for differentially private training of deep learning models, and \citet{Nasr2018} proposes adversarial regularization. 
However, training \emph{accurate} and privacy-preserving models is still a challenging research problem. Besides, the effect of these techniques (notably the randomness they induce) on model transparency is unknown. 
Finally, designing safe transparency reports is an important research direction: one needs to release explanations that are both safe and formally useful. 
For example, releasing no explanation (or random noise) is guaranteed to be safe, but is not useful; example-based methods are useful but cannot be considered safe. Quantifying the quality/privacy trade-off in model explanations will help us understand the capacity to which one can explain model decisions while maintaining data integrity.

%% file: RecordExperiments.tex
\section{The Influence of the Input Dimension}\label{app:capacity}
In Figure~\ref{fig:cap} we report experiments on the influence of the input dimension with networks of smaller/larger capacity, which have qualitatively similar behavior to our baseline model. However, the interval of $n$ in which correlation exists and the amount of correlation varies.

\begin{figure}[ht]
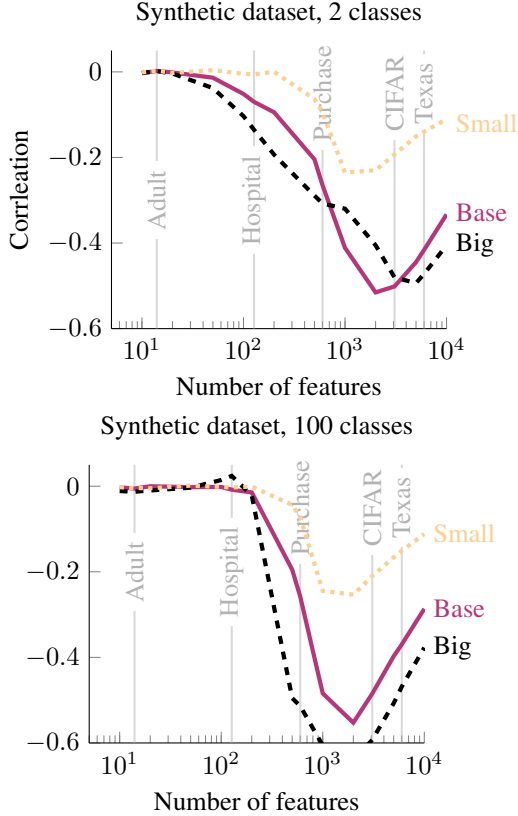

	\centering
	\if\compileFigures1
	\input{figure_scripts/fig_changing_capacity}
	\else
	\includegraphics[]{fig/\filename-figure\thefiguerNumber.pdf}
	\stepcounter{figuerNumber}
	\includegraphics[]{fig/\filename-figure\thefiguerNumber.pdf}
	\stepcounter{figuerNumber}
	\fi
	\caption{\small The correlation between $||\nabla c(\vec x)||_1$ and training membership for synthetic datasets for increasing number of  features $n$ and different number of classes $k \in \{ 2, 100 \}$ for three different networks. The "Small" has one hidden layer with 5 nodes, "Base" has two layers with 50 nodes each, "Big" has 3 layers with 100 nodes each. }
	\label{fig:cap}
\end{figure}

\section{Example-Based Explanations}\label{sec:influence-functions}
To illustrate the privacy risk of example-based model explanations, we focus on the approach proposed by \citet{koh2017understanding}. It aims at identifying influential {\em datapoints}; that is, given a point of interest $\poi$, find a subset of points from the training data $\phi(\poi)\subseteq \cal X$ that explains the label $c_{\hat{\theta}}(\poi)$, where $\hat \theta$ is a parameterization induced by a training algorithm $A$. 
It selects a training point $\vec x_\train$ by measuring the importance of $\vec x_\train$ for determining the prediction for $\poi$.\newline
To estimate the effect of $\vec x_\train$ on $\poi$, the explanation measures the difference in the loss function over $\poi$ when the model is trained with and without $\vec x_\train$. Let 
$
\theta_{\vec x_\train} \triangleq  A(\mathcal{X} \backslash \{\vec x _\train \})
$
, in words, $\theta_{\vec x_\train}$ is induced by training algorithm $A$ given the dataset excluding $\vec x_\train$. The influence of $\vec x_\train$ on $\poi$ is then 
\begin{align}
I_{\poi}(\vec x_\train) \triangleq  L(\poi,\theta_{\vec x_\train}) - L(\poi,\hat{\theta}). 
\label{eq:influence-value}
\end{align}
The \citeauthor{koh2017understanding} explanation releases the $k$ points with the highest absolute influence value according to Equation \eqref{eq:influence-value}. 
Additionally, it might release the influence of these $k$ points (the values of $I_{\poi}(\vec z)$ as per Equation~\eqref{eq:influence-value}), which allows users to gauge their relative importance.

\section{Membership-inference via Example-Based Model Explanations}\label{sec:record}

In this section we analyze how many training points an attacker can recover with access to example-based explanations. We focus on logistic regression models, for which example-based explanations were originally used \cite{koh2017understanding}. The results can be generalized to neural networks by focusing only on the last layer, we demonstrate this by considering a binary image classification dataset used in \cite{koh2017understanding}; we call this dataset \fishdog.  We also focus only on binary classification tasks. 
While technically the approaches discussed in previous sections could be applied to this setting as well, example-based explanations allow for stronger attacks. Specifically, they explicitly reveal training points, and so offer the attacker certainty about a points' training set membership (i.e. no false positives). Formally, we say a point $\vec y$ \emph{reveals} point $\vec x$ if for all $z \in \cal X$, 
$
|I_{\vec y} (\vec x)| \geq |I_{\vec y} (\vec z)| 
$.
In other words, $\vec x$ will be offered if one requests a example-based explanation for $\vec y$. 
Similarly, $\vec y$ \emph{$k$-reveals} point $\vec x$ if there is a subset $S \subseteq \cal{X}, |S| = k-1$ such that  $\forall z \in \cal X\backslash S\colon$  
$
|I_{\vec y} (\vec x)| \geq |I_{\vec y} (\vec z)|. 
$
Hence, $\vec x$ will be one of the points used to explain the prediction of $\vec y$ if one releases the top $k$ most influential points.
A point $\vec x \in \mathcal X$ that ($k$-)reveals itself, is called \emph{($k$-)self-revealing}.

\paragraph{Revealing membership for example-based explanations}\label{sec:record-based}
While for feature-based explanations the attacker needs to rely on indirect information leakage to infer membership, for example-based explanations, the attacker's task is relatively simple. Intuitively, a training point should be influential for its own prediction, so an example-based explanation of a training set member is likely to contain the queried point, revealing membership. We test this hypothesis via experiments, i.e. for every point $\vec x \in \mathcal{X}$ in the training set we obtain the $k \in \{ 1, 5,10\}$ most influential points for the prediction $f_\theta(\vec x)$ and see if $\vec x$ is one of them. 

\subsection{Experimental setup} 
While the theoretical framework of influence functions described in Section~\ref{sec:preliminaries} can be applied to an arbitrary classification task, it requires the training of as many classifiers as there are points in the training set in practice. \citet{koh2017understanding} propose an approximation method, but currently we only have access of its implementation for binary logistic regression models. 
However, for logistic regression, retraining the model (i.e. computing an {\em exact solution}) is actually faster than running the approximation algorithm. For larger models and datasets, these explanation methods seems intractable at the moment.     
This limits our experiments to the \adult and \hospital datasets from our previous experiments for which we train binary logistic regression models.  
Furthermore, given the relatively long time it takes to compute all models, we reduce the size of the training set to 2,000 points, so we can run one experiment within a few hours (see Figure~\ref{fig:varyDatasetSize} for an exploration of the effect of dataset size). 
\citet{koh2017understanding} use a specifically created dataset containing 2400 $299\times299$-pixel dog and fish images, extracted from ImageNet \cite{Russakovsky2015}. 
The images were pre-processed by taking the output of the penultimate layer of a neural network trained on ImageNet.\footnote{Specifically, the authors used an Inceptionv3 architecture, for which pre-trained models are available for Keras \url{https://keras.io/applications/\#inceptionv3}.} 
These latent representations were then used to train a linear regression model. 
For more variety we include this dataset here as the \emph{\fishdog} dataset. Arguably, this dataset doesn't contain particularly private information, but it can be seen as representative of image data in general. It also allows us to attack a pre-trained deep neural network which last layer was fine tuned. This type of transfer learning for small dataset becomes increasingly popular.  Given the small size of this dataset we randomly split it into 1800 training and 600 test points for each experiment.
For each dataset we repeat the experiment 10 times.
\begin{figure}
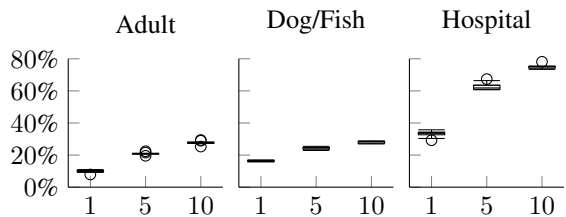

	\centering
	\if\compileFigures1
	\include{figure_scripts/fig_record_membership}
	\else
	\includegraphics[]{fig/\filename-figure\thefiguerNumber.pdf}
	\stepcounter{figuerNumber}
	\hspace{-0.3cm}
	\includegraphics[]{fig/\filename-figure\thefiguerNumber.pdf}
	\stepcounter{figuerNumber}
	\hspace{-0.3cm}
	\includegraphics[]{fig/\filename-figure\thefiguerNumber.pdf}
	\stepcounter{figuerNumber}
	\fi
	\vspace{-0.6cm}
	\caption{\small \% of training points revealed as their own explanation, when $k \in \{1,5,10 \}$ most influential points are revealed.}
	\label{fig:record_membership}
\end{figure}

\begin{figure}
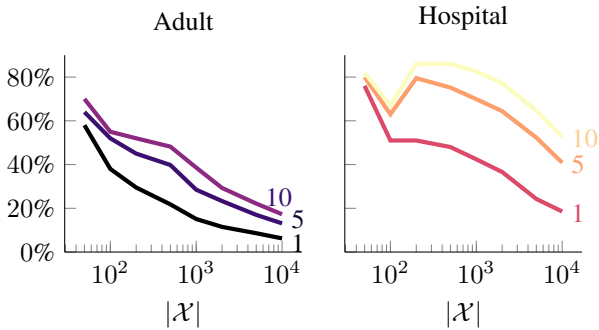

	\centering
	\if\compileFigures1
	\input{figure_scripts/fig_varyDatasetSize}
	\else
	\includegraphics[]{fig/\filename-figure\thefiguerNumber.pdf}
	\stepcounter{figuerNumber}
	\includegraphics[]{fig/\filename-figure\thefiguerNumber.pdf}
	\stepcounter{figuerNumber}
	\fi
	\caption{\small \% of $k$-self-revealing points depending on the size of the dataset ($|\mathcal{X}|$) for $k \in \{ 1, 5,10\}$ Increasing datasize the relative number of self-revealing points decreases. However, the decrease is relatively slow (the figure is in logarithmic scale) and the absolute number of revealed points is actually increasing.}
	\label{fig:varyDatasetSize}
\end{figure}

\subsection{Evaluation} Figure~\ref{fig:record_membership} shows the percentage of training points that would be revealed by explaining themselves. 
For the standard setting where the top 5 most influential points are revealed, a quarter of each dataset is revealed on average. 
For the \hospital dataset, two thirds of training points are revealed. 
Even when just the most influential point would be released for the \adult dataset (which exhibits the lowest success rates), 10\% of the members are revealed through this simple test. 

As mentioned in Appendix~\ref{sec:influence-functions}, the influence score of the most influential points might be released to a user as well. 
In our experiments, the influence scores are similarly distributed between training and test points (i.e. members and non-members);
however, the distribution is significantly  different once we ignore the
revealed training points. Figure~\ref{fig:record_membership_hist} illustrates this for one instance of the \fishdog dataset; similar results hold for the other datasets. An attacker can exploit these differences, using
techniques similar to those discussed in Section~\ref{sec:thresholdDescription}; however, we focus on other attack models in this work.

\begin{figure}
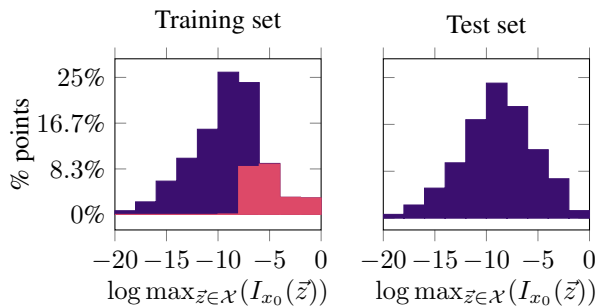

	\centering
	\if\compileFigures1
	\input{figure_scripts/fig_record_membership_hist}
	\else
	\includegraphics[]{fig/\filename-figure\thefiguerNumber.pdf}
	\stepcounter{figuerNumber}
	\includegraphics[]{fig/\filename-figure\thefiguerNumber.pdf}
	\stepcounter{figuerNumber}
	\fi
	\caption{\small Histogram of the influence of the most influential points for every point in the training set(left) and test set (right) on a logarithmic scale (for one instance of \fishdog). The points in the training set for which the membership inference is successful (i.e. $\vec{x}_0 = \argmax_{\vec{z} \in \mathcal{X}}(I_{x_0}(\vec{z}))$) are highlighted in red (dark).
	}
	\label{fig:record_membership_hist}
\end{figure}  

%% file: appendices/Appendix_MembershipInference_ExampleBased.tex
\subsection{Minority and outlier vulnerability to inference attacks for example-based explanations}\label{app:recMem} Visual inspection of datapoints for which membership attacks were successful indicates that outliers and minorities are more susceptible to being part of the explanation. Images of animals (a bear, a bird, a beaver) eating fish (and labeled as such) were consistently revealed (as well as a picture containing a fish as well as a (more prominent) dog that was labeled as fish). 
We label three ``minorities'' in the dataset to test the hypothesis that pictures of minorities are likelier to be revealed (Table~\ref{tab:minFishDog}).

\begin{table}
	\centering
	\begin{subtable}{\columnwidth}
		\centering
		\begin{tabular}{l c c c c}
			\toprule
			&\#points & $k=1$ & $k=5$ & $k=10$ \\
			\midrule
			Whole dataset       	& 100\%	&  16\% &  24\% &  28\% \\
			Birds               	& 0.6\%	&  64\% &  84\% &  86\% \\
			Clownfish           	& 1\%	&  14\% &  30\% &  33\% \\
			Lion fish           	& 1\%	&   9\% &  31\% &  43\% \\
			\bottomrule
		\end{tabular}  
		\caption{Disclosure likelihood by type in the \fishdog dataset.}
		\label{tab:minFishDog}
	\end{subtable}
	\begin{subtable}{\columnwidth}
		\centering
		\begin{tabular}{l c c c c}
			\toprule
			& \% of data & $k=1$ & $k=5$ & $k=10$ \\
			\midrule
			Whole dataset       	& 100\%	&  33\% &  63\% &  75\% \\
			Age 0 -10           	& 0.2\%	&  61\% & 100\% & 100\% \\
			Age 10 -20          	& 0.7\%	&  23\% &  62\% &  95\% \\
			Caucasian           	& 75\%	&  32\% &  62\% &  74\% \\
			African Amer.    	& 19\%	&  36\% &  66\% &  78\% \\
			Hispanics           	& 2\%	&  40\% &  61\% &  77\% \\
			Other race          	& 1.5\%	&  30\% &  60\% &  79\% \\
			Asian Amer.     	& 0.6\%	&  33\% &  70\% &  95\% \\
			
			\bottomrule
		\end{tabular}  
		\caption{Disclosure likelihood by age and race in the \hospital dataset.}
		\label{tab:minHospital}
	\end{subtable}
	\begin{subtable}{\columnwidth}
		\centering
		\begin{tabular}{l c c c c}
			\toprule
			& \% of data & $k=1$ & $k=5$ & $k=10$ \\
			\midrule
			Whole dataset       	& 100\%	&  10\% &  21\% &  28\% \\
			Age 10 -20          	& 5\%	&   0\% &   1\% &   1\% \\
			White               	& 86\%	&  10\% &  21\% &  28\% \\
			Black               	& 10\%	&   9\% &  16\% &  19\% \\
			A.-I.-E.  	& 1\%	&   8\% &  23\% &  32\% \\
			Other               	& 0.8\%	&  16\% &  32\% &  43\% \\
			A.-P.-I.  	& 3.\%	&  21\% &  40\% &  48\% \\
			\bottomrule
		\end{tabular}  
		\caption{Disclosure likelihood by age and race in the \adult dataset (A.-I.-E.: Amer-Indian-Eskimo; A.-P.-I.:Asian-Pac-Islander ).}
		\label{tab:minAdult}
	\end{subtable}
	\caption{\small Minority populations are more vulnerable to being revealed by the \citeauthor{koh2017understanding} method.}\label{tab:minorities}
\end{table}
With the exception of $k = 1$  (for lion and clown fish), minorities are likelier to be revealed. While clownfish (which are fairly ``standard'' fish apart from their distinct coloration) exhibit minor differences from the general dataset, birds are more than three times as likely to be revealed.
The \hospital dataset exhibits similar trends (Table~\ref{tab:minHospital}).
Young children, which are a small minority in the dataset, are revealed to a greater degree; ethnic minorities also exhibit slightly higher rates than Caucasians. This is mirrored (Table~\ref{tab:minAdult}) for the \adult dataset, with the exception of the age feature and the ``Black'' minority. Young people are actually particularly safe from inference. Note however, that young age is a highly predictive attribute for this classification. Only 3 out of the 2,510 entries aged younger than 20 have an income of more than 50K and none of them made it in any of our training set. Similarly only 12\% of the Black people in the dataset  (vs. 25\% over all) have a positive label, making this attribute more predictive an easier to generalize. 

While our findings are preliminary, they are quite troubling in the authors' opinion: 
transparency reports aim, in part, to {\em protect} minorities from algorithmic bias; however, data minorities are exposed to privacy risks through their implementation.
Our findings can be explained by earlier observations that training set outliers are likelier to be ``memorized'' and thus less generalized \cite{Carlini2018}; however, such memorization leaves minority populations vulnerable to privacy risks.

%% file: RecordAnalysis.tex
\section{Theoretical Analysis For Dataset Reconstruction}\label{sec:record_analysis}
Let us now turn our attention to a stronger type of attack; rather than inferring the membership of specific datapoints, we try to recover the {\em training dataset}.

Given a parameter $\theta = (\vec w, b) \in \R^n\times \R $ a logistic regression model is defined as
\[
f_\theta (\vec x)= \frac{1}{1+\exp^{-\vec w^T \vec x +b }}. 
\]
For a point $\vec y \in R^n$ with label $l(\vec y) \in \{0,1\}$ we define the loss of $f(\vec y)_\theta$ as 
\[
L(\vec y,\theta) = (1-f_\theta(\vec y))^{l(\vec z)}(f_\theta(\vec y))^{1-l(\vec z)},
\]
which corresponds to the standard maximal likelihood formulation. 
As logistic regression does not allow for a closed form optimum, it is generally trained via gradient ascent. 
We assume we are given a fixed training regime $A$ (i.e. fixed parameters for number of steps, learning rate etc.). 

Let $f_\theta$ be the model induced by $\mathcal X$ and 
$$\mathcal{F}_{\mathcal{X}} = \{f_{\theta_{\vec x}}| f_{\theta_{\vec x}} \text{ is induced by } \mathcal{X}\backslash \{ \vec x \}, \vec x \in \mathcal{X} \}$$ 
be the set of functions induced by omitting the training points. 
We can reformulate the influence of point $\vec x$ on point $\vec y$ (assuming $l(\vec y ) = 0$) as
\begin{align*}
	I_{\vec y} (\vec x) &= L(y,\theta) - L(y,\theta_{\vec x}) \\ 
	&= f_\theta(y) - f_{\theta_{\vec x}}(y) \\ 
	&=  \frac{1}{1+\exp^{-\vec w^T \vec y +b }} -  \frac{1}{1+\exp^{-\vec w_{\vec x}^T \vec y +b_{\vec x }}}
\end{align*}
The condition that $\vec y$ reveals $\vec x$ is thus equivalent to ensuring that for all $\vec z \in \cal X$, $|I_{\vec y} (\vec x)| \geq |I_{\vec y} (\vec z)|$. In the case of linear regression this simply implies that
\begin{align*}
	|\frac{1}{1+\exp^{-\vec w^T \vec y +b }} -  \frac{1}{1+\exp^{-\vec w_{\vec x}^T \vec y +b_{\vec x }}}|\\
	\geq |\frac{1}{1+\exp^{-\vec w^T \vec y +b }} -  \frac{1}{1+\exp^{-\vec w_{\vec z}^T \vec y +b_{\vec z}}}|
\end{align*}
which can be simplified to a linear constraint (whose exact form depends on whether the terms in absolute values are positive or negative).

\input{appendices/Appendix_lowerBoundRecord}

\subsubsection{Algorithm to recover the training set by minimizing the influence of already discovered points}\label{app:reconstructionAlg}
We construct a practical algorithm with which the attacker can iteratively reveal new points. 
Algorithm~\ref{alg:linearReconstruction}  consists of two main steps 
\begin{inparaenum}[(1)] 
	\item Sample a point in the current subspace and 
	\item find a new subspace in which all already discovered points have zero influence, and continue sampling in this subspace.
\end{inparaenum}

\begin{algorithm}[b!]
	\begin{algorithmic}[1] 
		\small
		\State $\linearlyIndependent \gets \True$
		\State $q  \gets 0$
		\State $R_0 \gets \R^n$
		\State $\revealedVectors \gets []$
		\While{\linearlyIndependent}
		\State $\vec y \gets \pickRandomPointIn(R_q)$
		\State $\vec x \gets \argmax_{\vec z \in \mathcal{X}}(|I_{\vec y}(\vec z)|)$ \label{line:argmax}
		\State $\revealedVectors\dotappend(\vec x)$
		\State $\vectorsNotFound \gets \getBasisOf(R_q)$ \label{line:BeginFind}
		\State $\spanningVectors \gets [\vec y ]$
		\While{$\vectorsNotFound\ne \emptyset$}
		\For{$\vec v \in \vectorsNotFound$} 
		\If{$\vec x = \argmax_{\vec z \in \mathcal{X}}(|I_{\vec y +\epsilon \vec v }(\vec z)|)$ } 
		\State $\spanningVectors\dotappend(\vec y  + \epsilon \vec v)$
		\EndIf
		\EndFor
		\State $\epsilon \gets \frac{\epsilon}{2}$
		\EndWhile \label{line:EndFind}
		\If{$\exists r \in \R\colon \forall \vec v \in \spanningVectors\colon |I_{\vec y +\epsilon \vec v }(\vec x)| = r $} 
		\State $\linearlyIndependent \gets \False$
		\Else 
		\State $R_{q+1} \gets \{\vec y \in R_q |\,| I_{\vec y}(\vec x)| = 0 \}$
		\State $q \gets q +1$
		\EndIf
		\EndWhile
		\State \Return \revealedVectors
	\end{algorithmic}
	\caption{\small Recovering the training set by minimizing the influence of already recovered points}
	\label{alg:linearReconstruction}
\end{algorithm}

 We note that for an small enough $\epsilon$ the algorithm needs $\sum_{i=0}^l n-i+1$  queries to reveal $l$ points. In our implementation we use $n$ queries per point to reconstruct each revealed point over $\mathbb{R}^n$ and instead of solving the set of equations exactly we solve a least squares problem to find regions with zero influence, which has greater numerical stability.   

Theorem~\ref{lemma:linearAlgorithm} offers a lower bound on the number of points that Algorithm \ref{alg:linearReconstruction} discovers. The theorem holds under the mild assumption on the sampling procedure, namely that when sampling from a subspace $R$ of dimension $k$, the probability of sampling a point within a given subspace $Q$ of dimension $<k$ is 0. This assumption is fairly common, and holds for many distributions, e.g. uniform sampling procedures. 
We say that a set of vectors $Q$ is $d$-wise linearly independent if every subset $D \subseteq Q$ of size $d$ is linearly independent. For a dataset $\mathcal X \subset \R^n$ and a training algorithm $A$, we say that the model parameters $\theta = (\vec w,b)$ are induced by $\mathcal{X}$ and write $(\vec{w},b) \triangleq  A(\mathcal{X})$.

\begin{restatable}{theorem}{linearAlgorithm}\label{lemma:linearAlgorithm}
Given a training set $\mathcal X \subset \R^n$, let $\theta = (\vec w,b)$ be the parameters of a logistic regression model induced by $\mathcal{X}$.  
Let $d \in \mathbb N$ be the largest number such that the vectors in 
\[
{W = \{ \vec w_{\vec x}| \vec x \in \mathcal{X}, (\vec{w}_{\vec x},b_{\vec x}) \triangleq  A(\mathcal{X} \backslash \{\vec x \})\}}
\] are $d$-wise linearly independent, then if all vectors in $W$ are linearly independent of $\vec w$, Algorithm~\ref{alg:linearReconstruction} (with query access to model predictions and example-based explanations) reveals at least $d$ points in $\cal X$ with probability 1.  	
\end{restatable}
\begin{proof}
	We need to show two statements \begin{inparaenum}[(a)] \item given an affine subspace $R_q$ of dimension $n-q$ in which the influence of the first $q<d$ points is zero, with probability 1 a new point is revealed and \item Algorithm~\ref{alg:linearReconstruction} constructs a new subspace $R_{q+1}$ in which the influence of the first $q+1$ points is zero, with probability 1.\end{inparaenum}
	
	For $q=0$ $R_q=\R^n$ and statement (a) is trivially true: querying any point will reveal a new point. 
	Now, when $q>0$ and the attacker queries a point $\vec y \in R_q$, the only reason that no new point is revealed is that for all points $\vec x \in \mathcal{X}$ we have $I_{\vec y}(\vec x) = 0$. Note that $I_{\vec y}(\vec x) = 0$ iff $-\vec w^T \vec y + b = -\vec w^T_{\vec x} \vec y + b_{\vec x}$. 
	Thus, if $I_{\vec y}(\vec x) = 0$ for all $\vec x \in \cal X$, then for all $\vec x_1, \vec x_2 \in \mathcal{X}$, $-\vec w^T_{\vec x_1} \vec y + b_{\vec x_1} = -\vec w^T_{\vec x_2} \vec y + b_{\vec x_2}$. 
	By the $d$-wise independence assumption of $\vec w_i$ this can only happen in an $n-d$-dimensional subspace of $Z \subset \R^n$. 
	The probability of sampling a point in the intersection of this subspace with $R_q$ is 0 as $\dim(R_q \cap Z) \leq n-d < n-q = \dim(R_q)$.
	
	To prove Statement~(b) we first note that the attacker can recover $\theta$ (i.e. the original function) by querying $f_\theta$ for $n+1$ suitable --- i.e., $n$ of these points must be linearly independent, and the last one is different but otherwise arbitrary point --- points and solving the corresponding linear equations. Knowing $\theta$, the attacker can recover $|\theta_{\vec x}|$ for a vector $\vec x \in \mathcal{X}$, given the values of $|I_{\vec y_i}(\vec x)|$ for $\vec y_i \in \R^n, i \in [n+1]$ suitable points. 
	Similarly, the attacker can reconstruct the entire behavior of $|I_{\vec y_i}(\vec x)|$ on an $n-q$ dimensional affine subspace $R_q$ from $n-q+1$ suitable points.  
	
	Lines~\ref{line:BeginFind} to \ref{line:EndFind} in Algorithm~\ref{alg:linearReconstruction} find these points. Since $|I_{\vec y}(\vec x)|$ is continuous, the $\argmax$ in Line~\ref{line:argmax} is constant in a small region around $\vec y$, as long as there is no tie. Hence, $\epsilon$ will eventually become sufficiently small so that all points $\vec y + \epsilon v$ reveal the same point as $\vec y$. 
	If there is a tie for the $\argmax$ at $\vec y$, let $S$ be a open set of points around $\vec y$ such that $\forall \vec y' \in S\colon |\argmax_{\vec z \in \mathcal{X}}(|I_{\vec y}(\vec z)|)|>1$. Now either $\dim S = n-q$ in which case there exists an $\epsilon$-ball $B$ that is a subset of $S$ around $y$ such that all points in $B$ reveal the same point as $\vec y$, or $\dim S < n-q$ in which case the probability of selecting $S$ or any point in it is zero. 
	So, with probability 1, Algorithm \ref{alg:linearReconstruction} finds $n-q+1$ suitable points that all reveal the same $\vec x \in \mathcal{X}$.  
	From this the attacker can determine  $|I_{\vec y}(\vec x)|$ for all points $\vec y \in R_q$. In particular, the attacker can recover $\{\vec y \in R_q : | I_{\vec y}(\vec x)| = 0 \}$  which, again by the assumption of $d$-wise independence, is an $n-(q+1)$-dimensional affine subspace of $R$.
\end{proof}
 
 Algorithm~\ref{alg:linearReconstruction} recovers large parts of the dataset for high-dimensional data. 
If $d \geq |\mathcal{X}|$, it can recover the entire dataset. For $d \ll |\mathcal{X}|$, only a small part of the dataset will be recovered. Algorithm~\ref{alg:linearReconstruction} is optimal in the sense that there are datasets for which only the points discovered by the algorithm can be discovered (see Figure~\ref{fig:bound_illustration} (Left)). However, there are many situations where the algorithm falls short of discovering all possible points (see Figure~\ref{fig:bound_illustration} (Right)). Furthermore, the algorithm does not exploit the fact that in practical situations the $k$ most influential points would be revealed, instead of just one. In fact, incorporating the information about the $k$-most influential points into Algorithm~\ref{alg:linearReconstruction} yields negligible benefits.   
The following heuristic offers no theoretical guarantees, but works well even when $d \ll |\mathcal{X}|$, as we show next. 

\subsubsection{Querying Revealed Points}\label{sec:heuristicAttack}
This heuristic is relatively simple: use previously revealed points to reveal new points. When querying a point the \citeauthor{koh2017understanding} measure returns $k$ training points (and their influence) as explanations. 
The attacker can now use these revealed points in new queries, obtaining new points from these queries. 
This naturally defines an \emph{influence graph} structure over the training set: Every point in the training set is a node $v$ in a graph $G$, with $k$ outgoing edges towards the $k$ nodes outputted by the \citeauthor{koh2017understanding} measure. 
The influence graph structure governs the degree to which one can adaptively recover the training set. 
For example, if the graph contains only one strongly connected component, an attacker can traverse (and recover) the entire training set from a single starting point. The following metrics are particularly informative:

\noindent\textbf{Number of strongly connected components (SCCs):} a high number of SCCs implies that the training set is harder to recover: an adaptive algorithm can only extract one SCC at a time. It also implies that the underlying prediction task is fragmented: labels in one part of the dataset are independent from the rest. 

\noindent\textbf{Size of the SCCs:} large SCCs help the attacker: they are more likely to be discovered, and recovering just some of them already results in recovery of significant portions of the training data.

\noindent\textbf{Distribution of in-degrees:} the greater a node's in-degree is, the likelier its recovery; for example, nodes with zero in-degree may be impossible for an attacker to recover. Generally speaking, a uniform distribution of in-degrees makes the graph easier to traverse.

%% file: appendices/Appendix_lowerBoundRecord.tex
\subsection{Bounds on number of revealable points}\label{app:lowerBoundRecord}
%This section contains an analysis on the upper and lower number of points that can be revealed through influence-based explanation methods revealing examples as explanations. 
It is relatively easy to construct examples in which only two points in the dataset can be revealed (see Figure~\ref{fig:bound_illustration}). In fact, there are specific instances in which only a single datapoint can be revealed (see Lemma~\ref{lemma:lowerBound}), however these cases are neither particularly insightful nor do they reflect real-world settings. On the other side, there exists datasets where, independent of the number of features, the entire dataset can be recovered. The right side of  Figure~\ref{fig:bound_illustration} illustrates such an example.

\begin{figure}[t!]
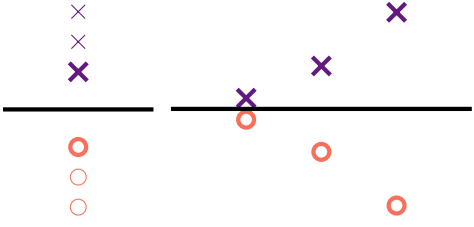

	\centering
	\if\compileFigures1
	\input{figure_scripts/fig_lower_bound_illustration}
	\input{figure_scripts/fig_upper_bound_illustration}
	\else
	\includegraphics[]{fig/\filename-figure\thefiguerNumber.pdf}
	\stepcounter{figuerNumber}
	\includegraphics[]{fig/\filename-figure\thefiguerNumber.pdf}
	\stepcounter{figuerNumber}
	\fi
	\caption{\small Illustrations of datasets for which only two (left) or all (right) points can be revealed under standard training procedures.}
	\label{fig:bound_illustration}
\end{figure}

The following Lemma characterizes the situations in which only a single point of the dataset can be revealed for $X \subseteq \R$. The conditions for higher dimensions follow from this.
\begin{lemma}\label{lemma:lowerBound}
	Given a training set $\mathcal X$ let $f_\theta$ and $\mathcal{F}_\mathcal{X}$ be induced by $\mathcal{X}$ with  $|\mathcal{F}_\mathcal{X}| \geq 2$, then one of the following statements is true
	\begin{enumerate}
		\item $\forall \vec x \in \mathcal{X}\colon w = w_{\vec x}$ and $ (b \geq b_{\vec{ x}}, \forall x \in \mathcal{X} ) \lor (b \leq b_{\vec{ x}}, \forall x \in \mathcal{X} ) $ (i.e all functions in $\mathcal F _\mathcal X$ are shifted in one direction of $f_\theta$),
		\item $\exists \vec y \in \R^n\colon \forall \vec x  \in \mathcal{X}\colon w_{\vec x}\vec y + b_{\vec x} = w\vec y + b $ (i.e all functions in $\mathcal{F}_\mathcal X$ intersect with  $f_\theta$ in the same point) ,
		\item $\forall \vec x \in \mathcal{X}\colon w = w_{\vec x}$ and there exists a numbering of the points in $\mathcal{X}$ such that $b_1 \leq b_2 \leq \dots \leq b_k \leq b \leq b_{k+1}, \dots , b_{m}$ such that 
		$b \leq  \log(\frac{1}{2}(e^{b_1}+e^{b_m})$.    
		\item at least two points can be revealed.
	\end{enumerate} 
\end{lemma}
\begin{proof}
	It is easy to see that in the first two situations only one point can be revealed (the one corresponding to the largest shift or largest angle at intersection). In the third case all functions in $\mathcal{F}_\mathcal{X}$ are shifts of $f_\theta$, but not all in the same direction. Only, the left most and right most shift are candidates for being revealed as they clearly dominate all other shifts. Also we assume $b_1<b<b_m$ (as soon as one shift coincidences with the original model the statement is trivially true). Some calculus reveals the condition for which the two points would have the same influence is
	\[
	y = \ln\left(\frac{-2e^b+e^{b_1}+e^{b_m}}{e^{b+b_1}+e^{b+b_m}-2e^{b_1+b_m}}\right)/w
	\],
	which is well defined when the expression inside the logarithm is positive and $e^{b+b_1}+e^{b+b_m}-2e^{b_1+b_m} \neq 0$. The former is the case for  $b < \log(\frac{1}{2}(e^{b_1}+e^{b_2})$, which also ensures the latter condition. On the other hand if $b \leq  \log(\frac{1}{2}(e^{b_1}+e^{b_m})$ the equation has no solution and so only one point can be revealed. 
	
	It remains to show that in all other cases at least two points can be revealed.
	Let's assume that there is a single $\vec x \in \mathcal{X}\colon \vec w \neq \vec w_{\vec x}$. In this case $ w_{\vec x}^T\vec y + b_{\vec x} = w^T\vec y + b $ can be solved and at the solutions $I(\vec x) = 0$ yet all other points have nonzero influence and one of them is revealed. Yet, since $ w_{\vec x'}^T\vec y + b_{\vec x} - w^T\vec y + b $ is constant for all $\vec x' \in \mathcal{X}, \vec x  \neq \vec x' $  and $| w_{\vec x}^T\vec y + b_{\vec x} - w^T\vec y + b |$  can take arbitrary values, there exists $\vec y \in R^n$ such that $\vec x = \argmax_{\vec z \in \mathcal{X}} |I_{\vec y}(\vec z)|$.
	Finally, if there are multiple points with  $\vec x \in \mathcal{X}\colon \vec w \neq \vec w_{\vec x}$ none of them can be revealed for all $\vec y$ as long as the condition in 2) is not satisfied. 
\end{proof} 

The following result states that there exist models and datasets for which every point can be revealed.
\begin{lemma}\label{lemma:upperBound}
	For every $m \in \mathbb{N}$ there exists a dataset $\cal X \subset \R^n$ with $|\cal X|=m$  and a training procedure $A$ so that any point in $\mathcal{X}$ can be revealed by example-based explanations.
\end{lemma}
\begin{proof}
	Given that we do not have restrictions on our training procedure $A$ the claim is equivalent to the existence of a parameter $\theta$ and a series of parameters $\theta_k$ and points $y_k$ such that $\forall k,m \in \mathbb{N}\colon$
	\[
	|L(\vec y_k, \theta) - L(\vec y_k, \theta_k)| = \max_{i \in [m]} |L(\vec y_k, \theta) - L(\vec y_k, \theta_i)| 
	\] 
	W.l.o.g. $n=1$, let $w = b = 0$ and $\forall k\colon l(y_k)=0$, then  
	\begin{align*}
	&|L(\vec y_k, \theta) - L(\vec y_k, \theta_k)| = \max_{i \in [m]} |L(\vec y_k, \theta) - L(\vec y_k, \theta_i)| \\
	\Leftrightarrow &L(\vec y_k, \theta_k) =  \max_{i \in [m]} |L(\vec y_k, \theta) - L(\vec y_k, \theta_i)\\
	\Leftrightarrow &f_{\theta_k}(\vec y_k) =  \max_{i \in [m]} f_{\theta_i}(\vec y_k) \Leftrightarrow w_k y_k -b_k  =  \max_{i \in [m]} w_i y_k -b_i
	\end{align*} 
	The above condition is satisfied for $w_k=2k$ , $y_k =k$ and $b_k=k^2$, as this describes the series of tangents of the strictly convex $x^2$ function.
\end{proof}

%% file: DatasetReconstruction.tex
\begin{figure}
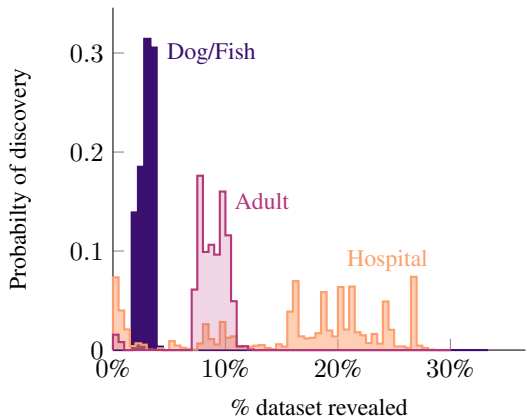

		\centering
		\if\compileFigures1
		\input{figure_scripts/fig_heuristic_attack}
		\else
		\includegraphics[]{fig/\filename-figure\thefiguerNumber.pdf}
		\stepcounter{figuerNumber}
		\fi
		\caption{Distribution over the size of the revealed training data, starting from a random point in the influence graph. This is obtained by averaging over $10$ initiations of the training set. Each time roughly the size of the largest SCC is recovered. For the \hospital dataset the size of the SCC varies the most, hence the multimodal distribution.}
		\label{fig:heuristicAttack}
\end{figure}
\begin{table}
		\centering
			\begin{tabular}{l l l l l}
				\toprule
				\multirow{2}{*}{Dataset} 	&\multirow{2}{*}{$n$} & \multirow{2}{*}{$|\mathcal{X}|$} & \# points  & \multirow{2}{*}{\% of $|\mathcal{X}|$} \\
				& & & recovered\\
				\midrule
				\fishdog 	&2048& 1,800 		   & 1790 & 99.4 \\
				\adult  	&104 & 2,000 		   & 91.5 & 4.6  \\
				\hospital   &127 & 2,000           & 81.1 & 4.0  \\
				\bottomrule
		\end{tabular}
		\caption{The number of points recovered using our attack based on subspace reduction. For small sized, high dimensional data, notably \fishdog, the attack can recover (almost) the entire dataset.}
		\label{tab:spaceReduction}
\end{table}
\begin{table}
			\begin{tabular}{l l l l}
				\toprule
				& \fishdog & \adult & \hospital  \\
				\midrule
				\#SCC 	 	 		& 1709	& 1742	& 1448	 \\
				\#SCC of size 1 	& 1679 	& 1701 	& 1333 	 \\
				Largest SCC size 	& 50 	& 167 	& 228 	 \\
				Max in-degree 		& 1069 	& 391 	& 287 	 \\
				\#node in-degree=0 	& 1364 	& 1568 	& 727 	 \\
				\bottomrule
		\end{tabular}
		\caption{Some key characteristics of the influence graphs induced by the example-based explanations (for $k = 5$). This is averaged over $10$ random initializations of the training set}
		\label{tab:keyGraphCharacteristics}
\end{table}

\section{Dataset Reconstruction for Example-Based Explanations}\label{sec:dataset-reconstruction}
We evaluate our reconstruction algorithms on the same datasets as the membership inference for example-based explanations. In this section we describe the results for the two approaches described above. In Appendix~\ref{app:RecordBaseline} we compare them to some general baselines.

\subsection{Attack based on subspace reduction}
Table~\ref{tab:spaceReduction} summarizes the results from our attack based on Algorithm~\ref{alg:linearReconstruction}. For \fishdog, we recover (nearly) the \emph{entire} dataset. The number of recovered points for \adult and \hospital is small compared to the size of the dataset, but especially for \adult close to the dimensionality of the data, which is the upper bound for the algorithm. In our actual implementation, rather than constructing the subspaces and sampling points in them we recovered the entire weight vectors and sampled points by solving a least squares optimization problem with some additional constraints to increase numerical stability. 

\subsection{Adaptive heuristic attack}
Table~\ref{tab:keyGraphCharacteristics} summarizes some key characteristics of the influence graphs for the training datasets (in this section we set $k=5$, larger values of $k$ lead to similar results). 
Remarkably, the influence graphs are fragmented, comprising of a large number of connected components, most of which have size 1. 
Each graph seems to have one larger SCC, which most of the smaller SCCs point to. This implies that an attacker starting at a random point in the graph is likely to recover only slightly more than the largest SCC. 
In fact, a significant amount of points (up to 75\% for the \adult  dataset) have in-degree 0: they cannot be recovered by this adaptive attack.
Figure~\ref{fig:heuristicAttack} illustrates the percentage of the dataset recovered when traversing the influence graph from a random starting point. 
For the \fishdog dataset, the largest SCC is revealed every time. For the \adult and \hospital datasets the largest SCC is not revealed for some points, resulting in limited success.

The order of queries does not affect the overall success of the attack, eventually every point reachable from the initial query will be revealed. However, minimizing the number of initial queries used is desirable.The number of queries required to traverse the discoverable part of the influence graph is stable under changes to the query schedule (e.g. BFS, DFS, random walk and several influence-based heuristics resulted in similar performance). 
In order to benchmark the performance of our attacker, who has no knowledge of the influence graph structure, we compare it to an omniscient attacker who knows the graph structure and is able to optimally recover the SCCs; this problem is known to be NP-complete (the best known constant factor approximation factor is 92)  \cite{daligault2009}. We thus compare our approach to a greedy omniscient attacker, which selects the node that is connected to the most unknown points.
Compared to this baseline, our approaches require roughly twice as many queries.

%% file: appendices/Appendix_RecordBaselineReconstruction.tex
\subsection{Baseline Attack}\label{app:RecordBaseline}
\begin{figure}
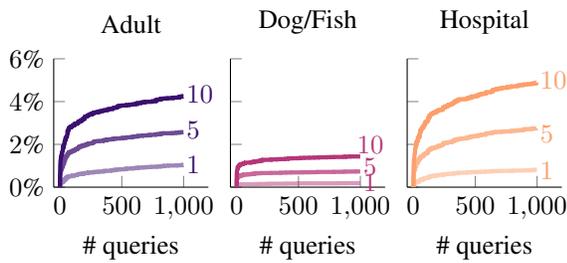
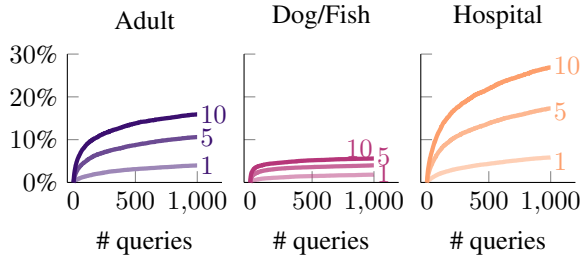
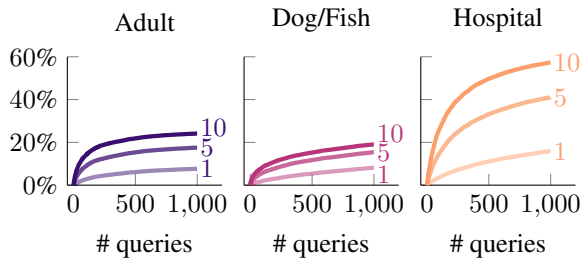

		\centering
	\begin{subfigure}[t]{\columnwidth}
		\centering		
		\if\compileFigures1
		\input{figure_scripts/fig_naive_dataset_reconstruction_uniform}
		\else
		\includegraphics[]{fig/\filename-figure\thefiguerNumber.pdf}
		\hspace{-0.4cm}
		\stepcounter{figuerNumber}
		\includegraphics[]{fig/\filename-figure\thefiguerNumber.pdf}
		\hspace{-0.4cm}
		\stepcounter{figuerNumber}
		\includegraphics[]{fig/\filename-figure\thefiguerNumber.pdf}
		\hspace{-0.4cm}
		\stepcounter{figuerNumber}
		\fi
		\caption{Uniform datapoint sampling. \label{fig:uniform}}
	\end{subfigure}
	%%%%%%%%%%%%%%%%%%%%%%%%%%%%%%%%%%%%%%%%%%%%%%%%%%%%%%
	\begin{subfigure}{\columnwidth}
		\centering
		\if\compileFigures1
		\input{figure_scripts/fig_naive_dataset_reconstruction_marginal}
		\else
		\includegraphics[]{fig/\filename-figure\thefiguerNumber.pdf}
		\hspace{-0.4cm}
		\stepcounter{figuerNumber}
		\includegraphics[]{fig/\filename-figure\thefiguerNumber.pdf}
		\hspace{-0.4cm}
		\stepcounter{figuerNumber}
		\includegraphics[]{fig/\filename-figure\thefiguerNumber.pdf}
		\hspace{-0.4cm}
		\stepcounter{figuerNumber}
		\fi
		\caption{Marginal feature distribution sampling (for the \fishdog dataset points are sampled in the latent space).\label{fig:marginal}}	
	\end{subfigure}
	\begin{subfigure}{\columnwidth}
		\centering
		\if\compileFigures1
		\input{figure_scripts/fig_naive_dataset_reconstruction_original}
		\else
		\includegraphics[]{fig/\filename-figure\thefiguerNumber.pdf}
		\hspace{-0.4cm}
		\stepcounter{figuerNumber}
		\includegraphics[]{fig/\filename-figure\thefiguerNumber.pdf}
		\hspace{-0.4cm}
		\stepcounter{figuerNumber}
		\includegraphics[]{fig/\filename-figure\thefiguerNumber.pdf}
		\hspace{-0.4cm}
		\stepcounter{figuerNumber}
		\fi
		\caption{True point distribution sampling.\label{fig:actual}}
	\end{subfigure}
	\caption{\small \% of training data revealed by an attacker using different sampling techniques, with $k \in \{1,5,10\}$ explanation points revealed per query.}
	\label{fig:naive-attacker}
\end{figure}
This section discusses some baseline attacks for reconstructing the target dataset. Each {\em baseline} attack model generates a static batch of transparency queries, i.e. new queries are not based on the attacker's past queries.
An attacker who has some prior knowledge on the dataset structure can successfully recover significant chunks of the training data; in what follows, we consider three different scenarios. 
\paragraph{Uniform samples}
With no prior knowledge on data distributions, an attacker samples points uniformly at random from the input space; this attack model is not particularly effective (Figure~\ref{fig:uniform}): even after observing 1,000 queries with 10 training points revealed per transparency query, less than 2\% of the \fishdog dataset and $\sim 3\%$ of the \hospital dataset are recovered. 
Moreover, the recovered images are unrepresentative of the data: since randomly sampled images tend to be white noise, the explanation images offered for them are those most resembling noise. 
\paragraph{Marginal distributions}
In a more powerful attack scenario, the attacker knows features' marginal distributions, but not the actual data distribution. 
In the case of images, the marginal distributions of individual pixels are rather uninformative; in fact, sampling images based on individual pixel marginals results in essentially random images. 
That said, under the Inception model, an attacker can sample points according to the marginal distribution of the latent space features: the weights for all nodes (except the last layer) are public knowledge; an attacker could reconstruct images using latent space sampling. Figure~\ref{fig:marginal}
shows results for the \hospital dataset, and the \fishdog dataset under the inception model. This attack yields far better results than uniform sampling; however, after a small number of queries, the same points tended to be presented as explanations, exhausting the attacker's capacity to reveal additional information.
\paragraph{Actual distribution}
This attack model offers access to the actual dataset distribution (we randomly sample points from the dataset that were not used in model training). 
This reflects scenarios where models make predictions on publicly available data. 
Using the actual data distribution, we can recover significant portions of the training data (Figure~\ref{fig:actual}). 
We again observe a saturation effect: additional queries do not improve the results significantly.